\documentclass[conference]{IEEEtran}
\IEEEoverridecommandlockouts
\pdfoutput=1
\usepackage{cite}
\usepackage{amsmath,amssymb,amsfonts}
\usepackage{algorithmic}
\usepackage{graphicx}
\usepackage{textcomp}
\usepackage{xcolor}
\usepackage[numbers]{natbib}
\usepackage{amsmath}
\usepackage{amsthm}
\usepackage{amssymb}
\usepackage{cite}
\usepackage{hyperref}
\usepackage{algorithm}
\usepackage{subfigure}
\usepackage{booktabs}
\usepackage{makecell}
\newcommand{\RNum}[1]{\uppercase\expandafter{\romannumeral #1\relax}}

\def\BibTeX{{\rm B\kern-.05em{\sc i\kern-.025em b}\kern-.08em
    T\kern-.1667em\lower.7ex\hbox{E}\kern-.125emX}}
\begin{document}

\title{Accelerating Federated Learning via\\
	Momentum Gradient Descent
}
\author{Wei Liu, Li Chen, Yunfei Chen, and Wenyi Zhang      
	
	\thanks{
		
		W. Liu, L. Chen and W. Zhang are with Department of Electronic Engineering and Information Science, University of Science and Technology of China. 
		(e-mail: liuwei93@mail.ustc.edu.cn, \{chenli87, wenyizha\}@ustc.edu.cn).
		Y. Chen is with the School of Engineering, University of Warwick, Coventry
		CV4 7AL, U.K. (e-mail: Yunfei.Chen@warwick.ac.uk).
}}
\maketitle

\begin{abstract}
Federated learning (FL) provides a communication-efficient approach to solve machine learning problems concerning distributed data, without sending raw data to a central server. However, existing works on FL only utilize first-order gradient descent (GD) and do not consider the preceding iterations to gradient update which can potentially accelerate convergence. In this paper, we consider momentum term which relates to the last iteration. The proposed momentum federated learning (MFL) uses momentum gradient descent (MGD) in the local update step of FL system. We establish global convergence properties of MFL and derive an upper bound on MFL convergence rate.
Comparing the upper bounds on MFL and FL convergence rate, we provide conditions in which MFL accelerates the convergence. For different machine learning models, the convergence performance of MFL is evaluated based on experiments with MNIST dataset.  Simulation results comfirm that MFL is globally convergent and further  reveal significant convergence improvement over FL. 
\end{abstract}

\section{Introduction}
Recently, data-intensive machine learning has been applied in various fields, such as autonomous driving \cite{chen2015deepdriving}, speech recognition \cite{deng2013recent}, image classification \cite{krizhevsky2012imagenet} and disease detection \cite{esteva2017dermatologist} since this technique provides beneficial solutions to extract the useful information hidden in data. It now becomes a common tendency that machine-learning
systems are deploying in architectures that include ten of thousands of
processors \cite{jordan2015machine}. 
Great amount of data is generated by various parallel and distributed physical objects.

Collecting data from edge devices to the central server is necessary for distributed machine learning scenarios. 
In the process of distributed data collection, there exist significant challenges such as energy efficiency problems and system latency problems.
The energy efficiency of distributed data collection was considered in wireless sensor networks (WSNs) due to limited battery capacity of sensors \cite{subramanian2006sleep}; 
In fifth-generation (5G) cellular networks, a round-trip delay from terminals through the network back to terminals demands much lower latencies, potentially down to 1 ms, to facilitate human tactile to visual feedback control \cite{fettweis2014tactile}.
Thus, the challenges of data aggregation in distributed system  urgently require communication-efficient solutions.

In order to overcome these challenges, cutting down transmission distance and reducing the amount of uploaded data from edge devices to the network center are two effective ways. To reduce transmission distance, mobile edge computing (MEC) in \cite{patel2014mobile} is an emerging technique where the computation and storage resources are pushed to proximity of edge devices where the local task and data offloaded by users can be processed. 
In this way, the distance of large-scale data transmission is greatly shortened and the latency has a significant reduction \cite{mao2017survey}. Using machine learning for the prediction of uploaded task execution time achieves a shorter processing delay \cite{hu2019learning}, and dynamic resource scheduling was studied to optimize resources allocation of MEC system in \cite{wang2018dynamic}.   
Moreover, using machine learning to offload computation for MEC can further reduce the computation and communication overhead \cite{yu2017computation}. 
To reduce the uploaded data size, model-based compression approaches, where raw data are compressed and represented by well-established model parameters, demonstrate significant compression performance \cite{hung2012evaluation}. Lossy compression  is also an effective strategy to decrease the uploaded data size \cite{di2018efficient}, \cite{di2017optimization}.    
Compressed sensing, where the sparse data of the edge can be efficiently sampled and reconstructed with transmitting a much smaller data size, was applied to data acquisition of Internet of Things (IoT) network \cite{yang2013wireless}, \cite{li2012compressed}. Further, the work in \cite{xu2018making} proposed a blockchain-based data sharing to enhance the communication efficiency across resource-limited edges. All the aforementioned works need to collect raw data from individual device.

To avoid collecting raw data for machine learning in distributed scenarios, a novel approach named \textit{Federated Learning} (FL) has emerged as a promising solution \cite{mcmahan2016communication}. 
The work in \cite{wang2018edge} provided a fundamental architecture design of FL.
Considering the growing computation capability of edge nodes (devices), FL decentralizes the centralized machine learning task and assigns the decomposed computing tasks to the edge nodes where the raw data are stored and learned at the edge nodes. After a fixed iteration interval, each edge node transmits its learned model parameter to the central server.
This strategy can substantially decrease consumption of communication resources and improve communication-efficiency. To further improve the energy efficiency of FL, the work in \cite{wang2019adaptive} proposed
an adaptive FL approach, where the aggregation frequency can be adjusted adaptively to minimize the loss function under a fixed resource budget.
To reduce the uplink communication costs, the work in \cite{konevcny2016federated} proposed structured and sketched updates method, and compression techniques were adopted to reduce parameter dimension in this work. 
A novel algorithm was proposed to compress edge node updates in \cite{hardy2017distributed}. In this algorithm, gradient selection and adaptive adjustment of learning rate were used for efficient compression. 
For security aggregation of high-dimensional data, the works in \cite{bonawitz2016practical}, \cite{bonawitz2017practical} provided a communication-efficient approach, where the server can compute the sum of model parameters from edge nodes without knowing the contribution of each individual node.  In \cite{nishio2018client}, under unbalanced resource distribution in
network edge, FL with client (edge node) selection was proposed for actively managing the clients aggregation according to their resources condition.
In \cite{zhao2018federated}, non-iid data distribution was studied and a strategy was proposed to improve FL learning performance under this situation.

However, existing FL solutions generally use gradient descent (GD) for loss function minimization. GD is a one-step method where the next iteration depends only on the current gradient. Convergence rate of GD can be improved by accounting for more preceding iterations \cite{nemirovsky1983problem}. Thus, by introducing the last iteration, which is named momentum term, momentum gradient descent (MGD) can accelerate the convergence. The work in \cite{polyak1964some} gave a theoretical explanation that under the assumption of strong convexity, MGD speeds up the convergence in comparison with GD. Further work in \cite{ghadimi2015global} provided conditions in which MGD is globally convergent.

Motivated by the above observation, we propose a new federated learning design of \textit{Momentum Federated Learning} (MFL) in this paper.
In the proposed MFL design, we introduce momentum term in  FL local update and leverage MGD to perform local iterations. 
Further, the global convergence of the proposed MFL is proven. We derive the theoretical convergence bound of MFL. Compared with FL \cite{wang2019adaptive}, the proposed MFL has an accelerated convergence rate under certain conditions. On the basis of MNIST dataset, we numerically study the proposed MFL and obtain its loss function curve. The experiment results show that MFL converges faster than FL for different machine learning models. The contributions of this paper are summarized as follows: 
\begin{itemize} 
	\item \textit{MFL design}: According to the characteristic that MGD facilitates machine learning convergence in the centralized situation, we propose MFL design where MGD is adopted to optimize loss function in local update. The proposed MFL can improve the convergence rate of distributed learning problem significantly.
	\item \textit{Convergence analysis for MFL}:
	We prove that the proposed MFL is globally convergent on convex optimization problems, and derive its theoretical  upper bound on convergence rate. 
	We make a comparative analysis of convergence performance between the proposed MFL and FL. It is proven that MFL improves convergence rate of FL under certain conditions.
	\item \textit{Evaluation based on MNIST dataset}:
	We evaluate the proposed MFL's convergence performance via simulation based on MNIST dataset with different machine learning models. Then an experimental comparison is made between FL and the proposed MFL. The simulation results show that MFL is globally convergent and
	confirm that MFL provides a significant improvement of convergence rate.
\end{itemize}

The remaining part of this paper is organized as follows. We introduce the system model to solve the learning problem in distributed scenarios in Section \RNum{2} and subsequently elaborate the existing solutions in Section \RNum{3}. In Section \RNum{4}, we describe the design of MFL in detail. 
Then in Section \RNum{5} and \RNum{6}, we present the convergence analysis of MFL and the comparison between FL and MFL, respectively.
Finally, we show experimentation results in Section \RNum{7} and draw a conclusion in Section \RNum{8}.

\begin{figure}[!t]
	
	\centering
	\includegraphics[scale=0.34]{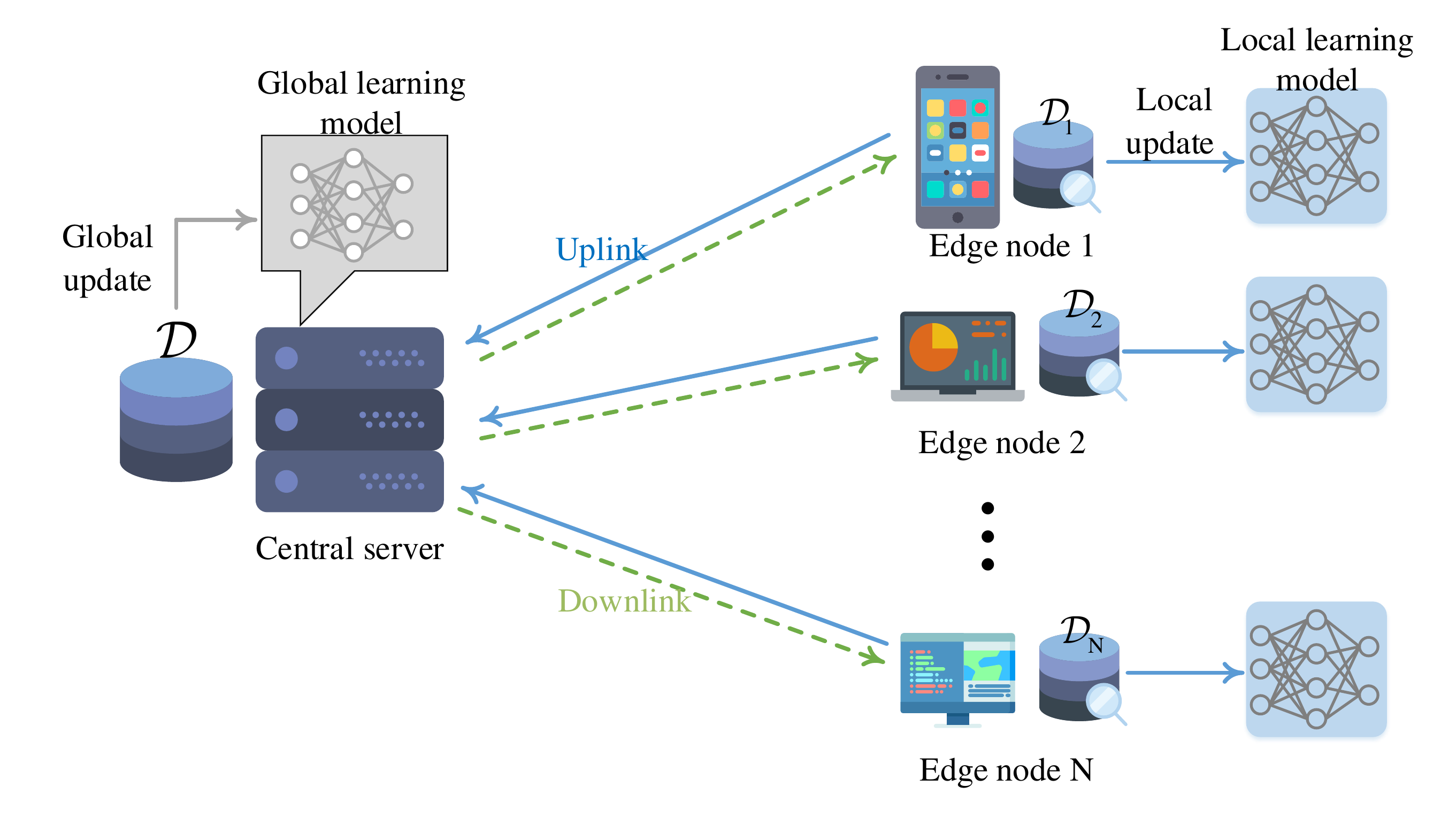}        
	\caption{The simplified structure of learning system for distributed user data}
	\label{fig2}
	
\end{figure}

\section{System Model}

In this paper, considering a simplified system model, we discuss the distributed network as shown in Fig. \ref{fig2}. This model has $N$ edge nodes and a central server.  
This $N$ edge nodes, which have limited communication and computation resources, contain local datasets $\mathcal{D}_1, \mathcal{D}_2,..., \mathcal{D}_i,..., \mathcal{D}_N$, respectively. So the global dataset is $\mathcal{D}\triangleq\mathcal{D}_1\cup\mathcal{D}_2\cup\cdots\cup\mathcal{D}_N$. Assume that $\mathcal{D}_i\cap\mathcal{D}_j=\emptyset$ for $i\neq j$. We define the number of samples in node $i$ as $|\mathcal{D}_i|$ where $|\cdot|$ denotes the size of the set. The total number of all nodes' samples is $|\mathcal{D}|$, and $|\mathcal{D}|=\sum_{i=1}^{N}|\mathcal{D}_i|$. 
The central server connects all the edge nodes for information transmission.  

We define \textit{the global loss function} at the central server as $F(\textbf{w})$, where $\textbf{w}$ denotes the model parameter. Different machine learning models correspond to different $F(\cdot)$ and $\textbf{w}$. We use $\textbf{w}^{*}$ to represent the optimal parameter for minimizing the value of $F(\textbf{w})$. Based on the presented model, the learning problem is to minimize $F(\textbf{w})$ and it can be formulated as follows:
\begin{align}\label{miniF}
\textbf{w}^{*}\triangleq \mathop {\arg} \mathop {\min} F(\textbf{w}).
\end{align}
Because of the complexity of machine learning model and original dataset, finding a closed-form solution of the above optimization problem is usually impossible. So algorithms based on gradient iterations are used to solve \eqref{miniF}. If raw user data are collected and stored in the central server, we can use centralized learning solutions to \eqref{miniF} while if raw user data are distributed over the edge nodes, FL and the proposed MFL can be applied to optimize this learning problem. 

Under the situation where FL or MFL solutions are used, the local loss function of node $i$ is denoted by $F_i(\textbf{w})$ which is defined merely on $\mathcal{D}_i$. Then we define the global loss function $F(\textbf{w})$ on $\mathcal{D}$ as follows:
\newtheorem{definition}{Definition}
\begin{definition}[Global loss function]\label{def1}
	Given the loss function $F_i(\textup{\textbf{w}})$ of edge node $i$, we define the global loss function on all the distributed datasets as
	\begin{align}\label{defF}
	F(\textup{\textbf{w}})\triangleq\frac{\sum_{i=1}^{N}|\mathcal{D}_i| F_i(\textup{\textbf{w}})}{|\mathcal{D}|}.
	\end{align}
\end{definition}

\section{Existing Solutions}

In this section, we introduce two existing solutions to solve the learning problem expressed by \eqref{miniF}. These two solutions are centralized learning solution and FL solution, respectively.

\subsection{Centralized Learning Solution}

Centralized machine learning is for machine learning model embedded in the central server and each edge node needs to send its raw data to the central sever. In this situation, edge nodes will consume communication resources for data transmission, but without incurring computation resources consumption.

After the central server has collected all datasets from the edge nodes, a usual way to solve the learning problem expressed by \eqref{miniF} is GD which globally converges for convex optimization problem. However, due to the one-step nature of GD that each update relates only to the current iteration, the convergence performance of GD still has the potential to be improved by accounting for the history of iterations. There exist multi-step methods where the next update relies not only on the current iteration but also on the preceding ones. MGD is the simplest multi-step method to solve the learning problem of \eqref{miniF} with convergence acceleration.
In MGD, each gradient update includes the last iteration which is called the momentum term. It has been proven that MGD improves convergence rate of GD in \cite{polyak1964some}.

\subsubsection{GD}

The update rule for GD is as follows:
\begin{align}\label{CGD}
\textbf{w}(t)=\textbf{w}(t-1)-\eta \nabla F(\textbf{w}(t-1)).
\end{align}
In \eqref{CGD}, $t$ denotes the iteration index and $\eta>0$ is the learning step size. The model parameter $\textbf{w}$ is updated along the direction of negative gradient. Using the above update rule, GD can solve the learning problem with continuous iterations.

\subsubsection{MGD}
As an improvement of GD, MGD introduces the momentum term and we present its update rules as follows:
\begin{align}
\label{d} \textbf{d}(t)&=\gamma \textbf{d}(t-1)+\nabla F(\textbf{w}(t-1))\\
\label{w} \textbf{w}(t)&=\textbf{w}(t-1)-\eta \textbf{d}(t),
\end{align}
where $\textbf{d}(t)$ is the momentum term which has the same dimension as $\textbf{w}(t)$, $\gamma$ is the momentum attenuation factor, $\eta$ is the learning step size and $t$ is the iteration index. By iterations of \eqref{d} and \eqref{w} with $t$, $F(\textbf{w})$ can potentially converge to the minimum faster compared with GD. The convergence range of MGD is $-1<\gamma<1$ with a bounded $\eta$ and 
if $0<\gamma<1$, MGD has an accelerated convergence rate than GD under a small $\eta$ typically used in simulations \citep[Result 3]{qian1999momentum}. 

\subsection{FL Solution}
Collecting and uploading the distributed data are challenging for communication-constrained distributed network.  Due to the limited communication resources at edge nodes, the transmission latency of raw user data could be large which deteriorates network feedback performance. Also, the transmission of user data is prone to be at the risk of violating user privacy and security over the entire network. In order to overcome these challenges, FL was proposed as a communication-efficient learning solution without data collection and transmission \cite{mcmahan2016communication}.


In contrast with centralized learning solutions, FL decouples the machine learning task from the central server to each edge node to avoid storing user data  in the server and reduce the communication consumption.  All of edge nodes make up a federation in coordination with the central server. In FL, each edge node performs local update and transmits its model parameter to the central server after a certain update interval. The central server integrates the received parameters to obtain a globally updated parameter, and sends this parameter back to all the edge nodes for next update interval.

The FL design and convergence analysis are presented in \cite{wang2019adaptive} where FL network is studied thoroughly. 
In an FL system, each edge node uses the same machine learning model. We use $\tau$ to denote the global aggregation frequency, i.e., the update interval. Each node $i$ has its local model parameter $\widetilde{\textbf{w}}_i(t)$, where the iteration index is denoted by $t=0,1,2,...$ (in this paper, an iteration means a local update). We use $[k]$ to denote the aggregation interval $[(k-1)\tau,k\tau]$ for $k=1,2,3,...$. 
At $t=0$, local model parameters of all nodes are initialized to the same value. When $t>0$, $\widetilde{\textbf{w}}_i(t)$ is updated locally based on GD, which is the \textit{local update}. After $\tau$ local updates, \textit{global aggregation} is performed and all edge nodes send the updated model parameters to the centralized server synchronously.

The learning process of FL is described as follows. 
\subsubsection{Local Update}
When $t\in[k]$, local updates are performed in each edge node by $$\widetilde{\textbf{w}}_i(t)=\widetilde{\textbf{w}}_i(t-1)-\eta\nabla F_i(\widetilde{\textbf{w}}_i(t-1)),$$ 
which follows GD exactly.
\subsubsection{Global Aggregation}
When $t=k\tau$, global aggregation is performed.
Each node sends $\widetilde{\textbf{w}}_i(k\tau)$ to the central server synchronously. 
The central server takes a weighted average of the received parameters from $N$ nodes to obtain the globally updated parameter $\textbf{w}(k\tau)$ by $$\textbf{w}(k\tau)=\frac{\sum_{i=1}^{N}|\mathcal{D}_i| \widetilde{\textbf{w}}_i(k\tau)}{|\mathcal{D}|}.$$ 
Then $\textbf{w}(k\tau)$ is sent back to all edge nodes as their new parameters and edge nodes perform local update for the next iteration interval.

In \citep[Lemma 2]{wang2019adaptive}, the FL solution has been proven to be globally convergent for convex optimization problems and exhibits good convergence performance. So FL is an effective solution to the distributed learning problem presented in \eqref{miniF}.
\begin{table}[!t]  
	\caption{MFL notation summary}
	\centering
	\label{table1}
	\renewcommand\arraystretch{1.5}
	\begin{tabular}{cp{6.4cm}}  
		
		\toprule[1pt]   
		
		\textbf{Notation} &\multicolumn{1}{c}{\textbf{Definition}}  \\  
		
		\midrule   
		
		$T$; $K$; $N$ & number of total local iterations; number of global aggregations/number of intervals; number of edge nodes     \\  
		
		$t$; $k$; $\tau$; $[k]$  & iteration index; interval index; aggregation frequency with $\tau=T/K$; the interval $[(k-1)\tau, k\tau]$  \\    
		
		$\textbf{w}^*$; $\textbf{w}^{\mathrm{f}}$&global optimal parameter of $F(\cdot)$; the optimal parameter that MFL can obtain in Algorithm \ref{alg1}\\
		
		$\eta$; $\beta$; $\rho$; $\gamma$ &the learning step size of MGD or GD; the $\beta$-smooth parameter of $F_i(\cdot)$; the $\rho$-Lipschitz parameter of $F_i(\cdot)$; the momentum attenuation factor which decides the proportion of momentum term in MGD \\
		
		$\mathcal{D}_i$; $\mathcal{D}$ &the local dataset of node $i$; the global dataset\\
		
		$\delta_i$; $\delta$ &the upper bound between $\nabla F(\textbf{w})$ and $\nabla F_i(\textbf{w})$; the average of $\delta_i$ over all nodes\\
		
		$F_i(\cdot)$; $F(\cdot)$& the loss function of node $i$; the global loss function \\ 
		
		$\textbf{d}(t)$; $\textbf{w}(t)$&  the global momentum parameter at iteration round $t$; the global model parameter at iteration round $t$\\
		
		
		$\widetilde{\textbf{d}}_i(t)$; $\widetilde{\textbf{w}}_i(t)$& the local momentum parameter of node $i$ at iteration round $t$; the local model parameter at iteration round $t$ \\
		
		$\textbf{d}_{[k]}(t)$; $\textbf{w}_{[k]}(t)$& the momentum parameter of centralized MGD at iteration round $t$ in $[k]$; the model parameter of centralized MGD at iteration round $t$ in $[k]$\\
		
		$\theta_{[k]}(t)$; $\theta$; $p$&the angle between vector $\nabla F(\textbf{w}_{[k]}(t))$ and $\textbf{d}_{[k]}(t)$; $\theta$ is the maximum of $\theta_{[k]}(t)$ for $1\leq k\leq K$ with $t\in[k]$; $p$ is the maximum ratio of $\|\textbf{d}_{[k]}(t)\|$ and $\|\nabla F(\textbf{w}_{[k]}(t))\|$ for $1\leq k\leq K$ with $t\in[k]$\\

		\bottomrule[1pt]  
		
	\end{tabular}
	
\end{table}

\section{Design of MFL}
In this section, we introduce the design of MFL to solve the distributed learning problem shown in \eqref{miniF}. We first discuss the motivation of our work. Then we present the design of MFL detailedly and the learning problem based on federated system. The main notations of MFL design and analysis are summarized in Table \ref{table1}.

\subsection{Motivation}
Since MGD improves the convergence rate of GD \cite{polyak1964some}, we want to apply MGD to local update steps of FL and hope that the proposed MFL will accelerate the convergence rate for federated networks. 
\begin{figure}[!t]
	
	\centering
	\includegraphics[scale=0.5]{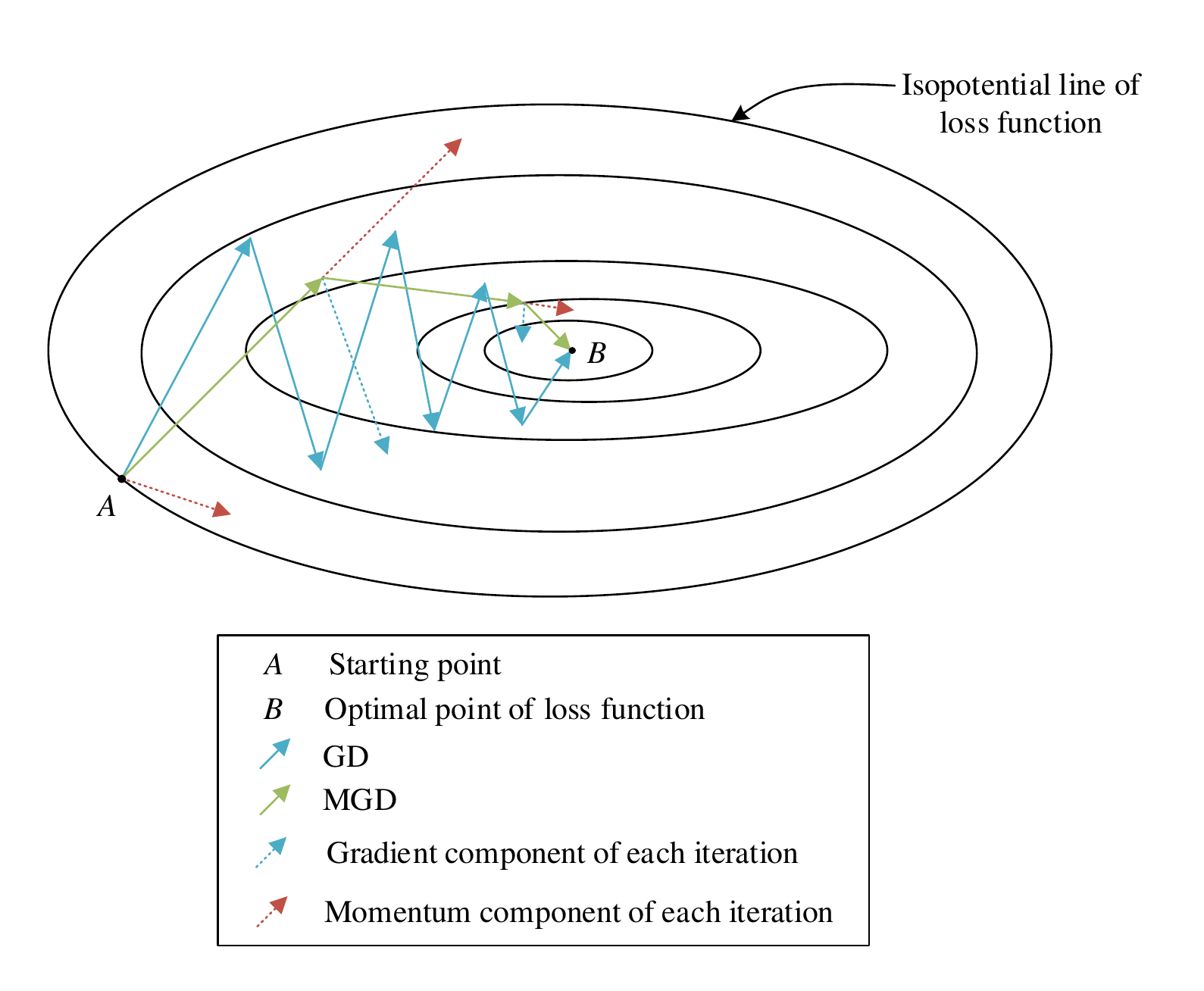}        
	\caption{Comparison of MGD and GD}
	\label{fig1}
	
\end{figure}

Firstly, we illustrate the intuitive influence on optimization problem after introducing the momentum term into gradient updating methods. 
Considering GD, the update reduction of the parameter is $\eta\nabla F(\textbf{w}(t-1))$ which is only proportional to the gradient of $\textbf{w}(t-1)$. The update direction of GD is always along gradient descent so that an oscillating update path could be caused, as shown by the GD update path in Fig. \ref{fig1}.
However, the update reduction of parameter for MGD is a superposition of $\eta\nabla F(\textbf{w}(t-1))$ and $\gamma(\textbf{w}(t-2)-\textbf{w}(t-1))$ which is the momentum term. As shown by the MGD update path in Fig. \ref{fig1}, utilizing the momentum term can deviate the direction of parameter update to the optimal decline significantly and mitigate the oscillation caused by GD. 
In Fig. \ref{fig1}, GD has an oscillating update path and costs seven iterations to reach the optimal point while MGD only needs three iterations to do that, which demonstrates mitigating the oscillation by MGD leads to a faster convergence rate. 

Because edge nodes of distributed networks are usually resource-constrained, solutions to convergence acceleration can attain higher resources utilization efficiency.
Thus, motivated by the property that MGD improves convergence rate, we use MGD to perform local update of FL and this approach is named MFL.

In the following subsections, we design the MFL learning paradigm and propose the learning problem based on the MFL design.

\subsection{MFL}



\begin{figure*}[!t]
	
	\centering
	\includegraphics[scale=0.4]{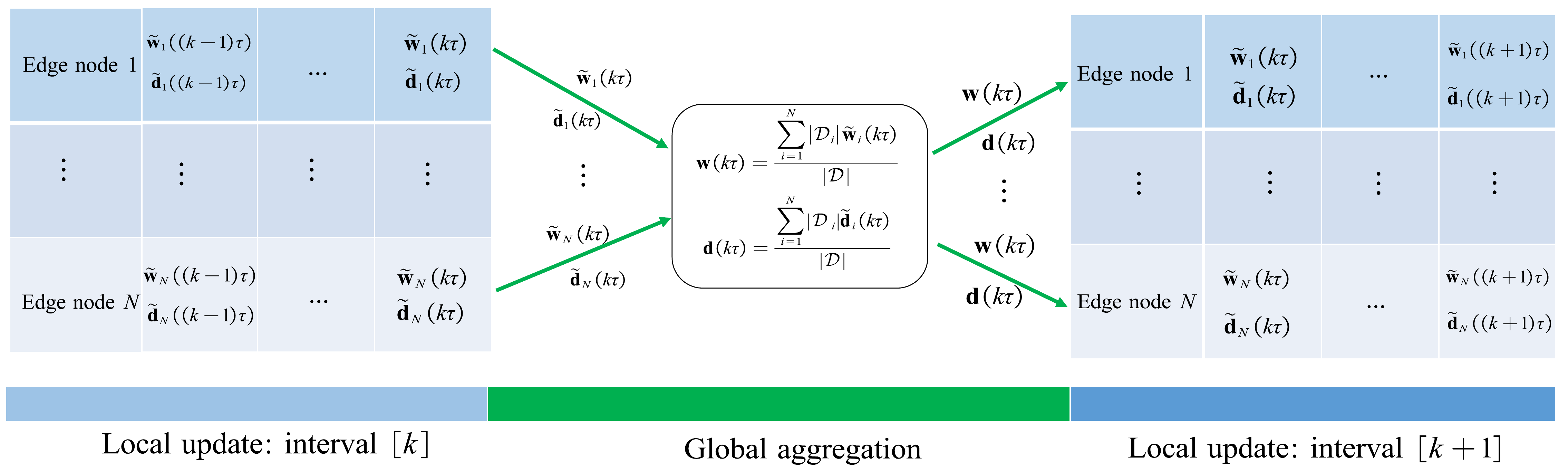}        
	\caption{Illustration of MFL local update and global aggregation steps from interval $[k]$ to $[k+1]$.} 
	\label{fig3}
	
\end{figure*}

In the MFL design, we use $\widetilde{\textbf{d}}_i(t)$ and $\widetilde{\textbf{w}}_i(t)$ to denote momentum parameter and model parameter for node $i$, respectively. 
All edge nodes are set to embed the same machine learning models. So the local loss functions $F_i(\textbf{w})$ is the same for all nodes, and the dimension of both the model parameters and the momentum parameters are consistent. 
The parameters setup of MFL is similar to that of FL.
We use $t$ to denote the local iteration index for $t=0,1,...$, $\tau$ to denote the aggregation frequency and $[k]$ to denote the interval $[(k-1)\tau,k\tau]$ where $k$ denotes the interval index for $k=1,2,...$. 
At $t=0$, the momentum parameters and the model parameters of all nodes are initialized to the same values, respectively. When $t\in[k]$, $\widetilde{\textbf{d}}_i(t)$ and $\widetilde{\textbf{w}}_i(t)$ are updated based on MGD, which is called \textit{local update steps}. When $t=k\tau$, MFL performs \textit{global aggregation steps} where $\widetilde{\textbf{d}}_i(t)$ and $\widetilde{\textbf{w}}_i(t)$ are sent to the central server synchronously. Then in the central server, the global momentum parameter $\textbf{d}(t)$ and the global model parameter $\textbf{w}(t)$ are obtained by taking a weighted average of the received parameters, respectively, and are sent back to all edge nodes for the next interval. 

The learning rules of MFL include the local update and  the global aggregation steps. By continuous alternations of local update and global aggregation, MFL can perform its learning process to minimize the global loss function $F(\textbf{w})$. We describe the MFL learning process as follows. 

First of all, we set  initial values for $\widetilde{\textbf{d}}_i(0)$ and $\widetilde{\textbf{w}}_i(0)$. Then 

\textit{1) Local Update}: When $t\in[k]$, local update is performed at each edge node by
\begin{align}
\label{di} \widetilde{\textbf{d}}_i(t)&=\gamma \widetilde{\textbf{d}}_i(t-1)+\nabla F_i(\widetilde{\textbf{w}}_i(t-1))\\
\label{wi} \widetilde{\textbf{w}}_i(t)&=\widetilde{\textbf{w}}_i(t-1)-\eta \widetilde{\textbf{d}}_i(t).
\end{align} 
According to \eqref{di} and \eqref{wi}, node $i$ performs MGD to optimize the loss function $F_i(\cdot)$ defined on its own dataset.

\textit{2) Global Aggregation}: When $t=k\tau$, node $i$ transmits $\widetilde{\textbf{w}}_i(k\tau)$ and $\widetilde{\textbf{d}}_i(k\tau)$ to the central server which takes weighted averages of the received parameters from $N$ nodes to obtain the global parameters $\textbf{w}(k\tau)$ and $\textbf{d}(k\tau)$, respectively. The aggregation rules are presented as follows:
\begin{align}
\label{dd} \textbf{d}(t)&=\frac{\sum_{i=1}^{N}|\mathcal{D}_i| \widetilde{\textbf{d}}_i(t)}{|\mathcal{D}|}\\
\label{ww} \textbf{w}(t)&=\frac{\sum_{i=1}^{N}|\mathcal{D}_i| \widetilde{\textbf{w}}_i(t)}{|\mathcal{D}|}.
\end{align}
Then the central server sends $\textbf{d}(k\tau)$ and $\textbf{w}(k\tau)$ back to all edge nodes where $\widetilde{\textbf{d}}_i(k\tau)=\textbf{d}(k\tau)$ and $\widetilde{\textbf{w}}_i(k\tau)=\textbf{w}(k\tau)$ are set to enable the local update in the next interval $[k+1]$. Note that only if $t=k\tau$, the value of the global parameters $\textbf{w}(t)$ and $\textbf{d}(t)$ can be observed. But we define $\textbf{d}(t)$ and $\textbf{w}(t)$ for all $t$ to facilitate the following analysis. A typical alternation are shown in Fig. \ref{fig3} which intuitively illustrates the learning steps of MFL in interval $[k]$ and $[k+1]$. 

\begin{algorithm}[!t]
	\caption{\textit{MFL} The dataset in each node has been set, and the machine learning model embedded in edge nodes has been chosen. We have set appropriate model parameters $\eta$ and $\gamma$.} 
	\label{alg1}
	\begin{algorithmic}[1]
		\REQUIRE ~~\\ 
		The limited number of local updates in each node $T$\\
		A given aggregation frequency $\tau$
		\ENSURE ~~\\ 
		The final global model weight vector $\textbf{w}^\mathrm{f}$
		\STATE Set the initial value of $\textbf{w}^\mathrm{f}$, $\widetilde{\textbf{w}}_i(0)$ and $\widetilde{\textbf{d}}_i(0)$. 
		\label{ code:fram:extract }
		\FOR{$t=1,2,...,T$}
		\label{code:fram:trainbase}
		\STATE Each node $i$ performs local update in parallel according to \eqref{di} and \eqref{wi}.$//$\textit{Local update} 
		\label{code:fram:add}
		\IF{$t==k\tau$ where $k$ is a positive integer} 
		\STATE	Set $\widetilde{\textbf{d}}_i(t)\gets\textbf{d}(t)$ and $\widetilde{\textbf{w}}_i(t)\gets\textbf{w}(t)$ for all nodes where $\textbf{d}(t)$ and $\textbf{w}(t)$ is obtained by \eqref{dd} and \eqref{ww} respectively. $//$\textit{Global aggregation}\\
		Update $\textbf{w}^\mathrm{f}\gets \mathop {\arg} \mathop {\min}_{\textbf{w}\in\{\textbf{w}^\mathrm{f},\textbf{w}(k\tau)\}} F(\textbf{w})$
		\ENDIF
		\ENDFOR
	\end{algorithmic}
\end{algorithm}
The learning problem of MFL to attain the optimal model parameter is presented as \eqref{miniF}. However, the edge nodes have limited computation resources with a finite number of local iterations. We assume that $T$ is the number of local iterations and $K$ is the corresponding number of global aggregations. Thus, we have $t\leq T$ and $k\leq K$ with $T=K\tau$. Considering that $\textbf{w}(t)$ is unobservable for $t\neq k\tau$, we use $\textbf{w}^\mathrm{f}$ to denote the achievable optimal model parameter defined on resource-constrained MFL network. Hence, the learning problem is to obtain $\textbf{w}^\mathrm{f}$ within $T$ local iterations particularly, i.e.,
\begin{align}
\label{limited_Fmin}\textbf{w}^\mathrm{f} \triangleq \mathop {\arg \min}_{\textbf{w}\in\{\textbf{w}(k\tau):k=1,2,...,K\}} F(\textbf{w}).
\end{align}
The optimization algorithm of MFL is explained in Algorithm \ref{alg1}. 

\section{Convergence Analysis}
In this section, we firstly make some definitions and assumptions for MFL convergence analysis. Then based on these preliminaries, global convergence properties of MFL following Algorithm \ref{alg1} are established and an upper bound on MFL convergence rate is derived. Also MFL convergence performance  with related parameters is analyzed.

\subsection{Preliminaries}
First of all, to facilitate the analysis, we assume that $F_i(\textbf{w})$ satisfies the following conditions:  

\newtheorem{assumption}{Assumption}
\begin{assumption}\label{ass1}
	For $F_i(\textup{\textbf{w}})$ in node $i$, we assume the following conditions:\\
	1) $F_i(\textup{\textbf{w}})$ is convex\\
	2) $F_i(\textup{\textbf{w}})$ is $\rho$-Lipschitz, i.e., $|F_i(\textup{\textbf{w}}_1)-F_i(\textup{\textbf{w}}_2)|\leq \rho\|\textup{\textbf{w}}_1-\textup{\textbf{w}}_2\|$ for some $\rho>0$ and any $\textup{\textbf{w}}_1$, $\textup{\textbf{w}}_2$\\
	3) $F_i(\textup{\textbf{w}})$ is $\beta$-smooth, i.e., $\|\nabla F_i(\textup{\textbf{w}}_1)-\nabla F_i(\textup{\textbf{w}}_2)\|\leq \beta \|\textup{\textbf{w}}_1-\textup{\textbf{w}}_2\|$ for some $\beta>0$ and any $\textup{\textbf{w}}_1$, $\textup{\textbf{w}}_2$\\
	4) $F_i(\textup{\textbf{w}})$ is $\mu$-strong, i.e., $aF_i(\textup{\textbf{w}}_1)+(1-a)F_i(\textup{\textbf{w}}_2)\geq F_i(a\textup{\textbf{w}}_1+(1-a)\textup{\textbf{w}}_2)+\frac{a(1-a)\mu}{2}\|\textup{\textbf{w}}_1-\textup{\textbf{w}}_2\|^2$, $a\in[0,1]$ for some $\mu>0$ and any $\textup{\textbf{w}}_1$, $\textup{\textbf{w}}_2$ \textup{\citep[Theorem 2.1.9]{nesterov2018lectures}}
\end{assumption}

Because guaranteeing the global convergence of centralized MGD requires that the objective function is strongly convex \cite{polyak1964some}, it is necessary to assume the condition 4.    
Assumption \ref{ass1} is satisfied for some learning models such as support vector machine (SVM), linear regression and logistic regression whose loss functions are presented in Table \ref{table_time}.
From Assumption \ref{ass1}, we can obtain the following lemma:
\newtheorem{lemma}{Lemma}
\begin{lemma} \label{lemma1}
	$F(\textup{\textbf{w}})$ is convex, $\rho$-Lipschitz, $\beta$-smooth and $\mu$-strong.
\end{lemma}
\begin{proof}
	According to the definition of $F(\textup{\textbf{w}})$ from \eqref{defF}, triangle inequality and the definition of $\rho$-Lipschitz, $\beta$-smooth and $\mu$-strong, we can derive that $F(\textup{\textbf{w}})$ is convex, $\rho$-Lipschitz, $\beta$-smooth and $\mu$-strong directly. 
\end{proof}

Then we introduce the gradient divergence between $\nabla F(\textbf{w})$ and $\nabla F_i(\textbf{w})$ for any node $i$. It comes from the nature of the difference in datasets distribution. 
\begin{definition}[Gradient divergence]
	We define $\delta_i$ as the upper bound between 
	$\nabla F(\textup{\textbf{w}})$ and $\nabla F_i(\textup{\textbf{w}})$ for any node $i$, i.e.,
	\begin{align}\label{grad_gap}
	\|\nabla F(\textup{\textbf{w}})-\nabla F_i(\textup{\textbf{w}})\|\leq \delta_i.
	\end{align}
	Also, we define the average gradient divergence 
	\begin{align}\label{ave_gra_gap}
	\delta\triangleq\frac{\sum_i |\mathcal{D}_i| \delta_i}{|\mathcal{D}|}.
	\end{align}
\end{definition}
\textbf{Boundedness of $\delta_i$ and $\delta$}:
Based on condition 3 of Assumption \ref{ass1}, we let $\textbf{w}_2=\textbf{w}_i^*$ where $\textbf{w}_i^*$ is the optimal value for minimizing $F_i(\textbf{w})$.  
Because $F_i(\textbf{w})$ is convex, we have $\|\nabla F_i(\textbf{w}_1)\|\leq \beta\|\textbf{w}_1-\textbf{w}_i^*\|$ for any $\textbf{w}_1$, which means $\|\nabla F_i(\textbf{w})\|$ is finite for any $\textbf{w}$. According to Definition \ref{def1} and the linearity of gradient operator, global loss function $\nabla F(\textbf{w})$ is obtained by taking a weighted average of $\nabla F_i(\textbf{w})$. Therefore, $\|\nabla F(\textbf{w})\|$ is finite, and $\|\nabla F(\textup{\textbf{w}})-\nabla F_i(\textup{\textbf{w}})\|$ has an upper bound, i.e., $\delta_i$ is bounded. Further, $\delta$ is still bounded from the linearity in \eqref{ave_gra_gap}.

Since local update steps of MFL perform MGD, the upper bounds of MFL and MGD convergence rate exist certain connections in the same interval.
For the convenience of analysis, we use variables $\textbf{d}_{[k]}(t)$ and $\textbf{w}_{[k]}(t)$ to denote the momentum parameter and the model parameter of \textit{centralized MGD} in each interval $[k]$, respectively. This centralized MGD is defined on global dataset and updated based on global loss function $F(\textbf{w})$. 
In interval $[k]$, the update rules of centralized MGD follow:
\begin{align}
\label{CMGDd}\textbf{d}_{[k]}(t)&=\gamma \textbf{d}_{[k]}(t-1)+\nabla F(\textbf{w}_{[k]}(t-1))\\
\label{CMGDv}\textbf{w}_{[k]}(t)&=\textbf{w}_{[k]}(t-1)-\eta \textbf{d}_{[k]}(t).
\end{align}
At the beginning of interval $[k]$, the momentum parameter $\textbf{d}_{[k]}(t)$ and the model parameter $\textbf{w}_{[k]}(t)$ of centralized MGD  are synchronized with the corresponding parameters of MFL, i.e., 
\begin{align*}
&\textbf{d}_{[k]}((k-1)\tau)\triangleq\textbf{d}((k-1)\tau)\\
&\textbf{w}_{[k]}((k-1)\tau)\triangleq\textbf{w}((k-1)\tau).
\end{align*}
For each interval $[k]$,
the centralized MGD is performed by iterations of \eqref{CMGDd} and \eqref{CMGDv}. In the Fig. \ref{interval}, we illustrate the distinctions  between $F(\textbf{w}(t))$ and $F(\textbf{w}_{[k]}(t))$ intuitively.

\begin{figure}[!t]
	
	\centering
	\includegraphics[scale=0.86]{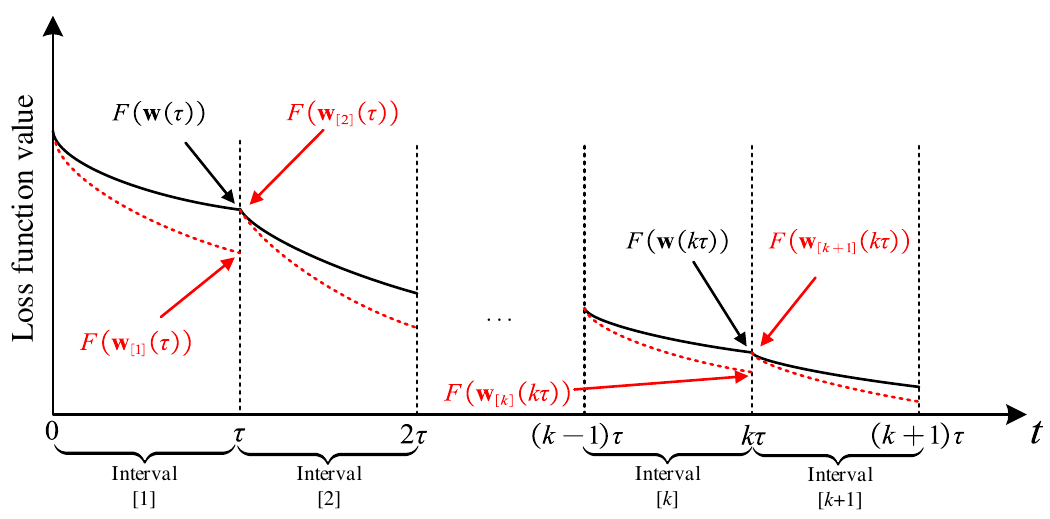}        
	\caption{Illustration of the difference between MGD and MFL in intervals}
	\label{interval}
	
\end{figure}

Comparing with centralized MGD, MFL aggregation interval with $\tau>1$ brings global update delay because of the fact that centralized MGD performs global update on every iteration while MFL is allowed to spread its global parameter to edge nodes after $\tau$ local updates. Therefore, the convergence performance of MFL is worse than that of MGD, which is essentially from the imbalance between several computation rounds and one communication round in MFL design.
The following subsection provides the resulting convergence performance gap between these two approaches.  


\newtheorem{thm}{\bf Theorem}
\subsection{Gap between MFL and Centralized MGD in Interval $[k]$}
Firstly, considering a special case, we consider the gap between MFL and centralized MGD for $\tau=1$. From physical perspective, MFL performs global aggregation after every local update and there does not exist global parameter update delay, i.e, the performance gap is zero. In Appendix \ref{A}, we prove that MFL is equivalent to MGD for $\tau=1$ from theoretical perspective.

Now considering general case for any $\tau \geq1$, the upper bound of gap between $\textbf{w}(t)$ and $\textbf{w}_{[k]}(t)$ can be derived as follows.
\newtheorem{proposition}{Proposition}
\begin{proposition}[Gap of MFL$\&$centralized MGD in intervals]\label{thm2}
	Given $t\in[k]$, the gap between $\textup{\textbf{w}}(t)$ and $\textup{\textbf{w}}_{[k]}(t)$ can be expressed by
	\begin{align}\label{gap}
	\|\textup{\textbf{w}}(t)-\textup{\textbf{w}}_{[k]}(t)\|\leq h(t-(k-1)\tau),
	\end{align}
	where we define 
	\begin{align*}
	A&\triangleq\frac{(1\!+\!\gamma+\!\eta\beta)\!+\!\sqrt{(1\!+\!\gamma\!+\!\eta\beta)^2\!-\!4\gamma}}{2\gamma},\\
	B&\triangleq\frac{(1\!+\!\gamma+\!\eta\beta)\!-\!\sqrt{(1\!+\!\gamma\!+\!\eta\beta)^2\!-\!4\gamma}}{2\gamma},	
	\end{align*}
	\begin{align*}
	E&\triangleq\frac{A}{(A-B)(\gamma A-1)},\\
	F&\triangleq\frac{B}{(A-B)(1-\gamma B)}
	\end{align*}
	and $h(x)$ yields 
	\begin{align}\label{hx}
	h(x)&\!=\!\eta\delta\left[ E (\gamma A)^x\!+\!F (\gamma B)^x\!-\!\frac{1}{\eta\beta}\!-\!\frac{\gamma(\gamma^x-1)-(\gamma-1)x}{(\gamma-1)^2}\right]
	\end{align}
	for $0<\gamma<1$ and any $x=0,1,2,...$. 
	
	Because $F(\textup{\textbf{w}})$ is $\rho$-Lipschitz from \textup{Lemma \ref{lemma1}}, it holds that 
	\begin{align}\label{gapff}
	F(\textup{\textbf{w}}(t))-F(\textup{\textbf{w}}_{[k]}(t))\leq \rho h(t-(k-1)\tau).
	\end{align}
\end{proposition}
\begin{proof}
	Firstly, we derive an upper bound of $\|\widetilde{\textbf{w}}_i(t)-\textbf{w}_{[k]}(t)\|$ for node $i$. On the basis of this bound, we extend this result from the local cases to the global one to obtain the final result. The detailed proving process is presented in Appendix \ref{B}.
\end{proof}

Because $h(1)=h(0)=0$ and $h(x)$ increases with $x$ for $x\geq1$, which are proven in Appendix \ref{C}, we always have $h(x)\geq0$ for $x=0,1,2,...$.

From Proposition \ref{thm2}, in any interval $[k]$,  
we have $h(0)=0$ for $t=(k-1)\tau$, which fits the definition $\textbf{w}_{[k]}((k-1)\tau)=\textbf{w}((k-1)\tau)$. 
We still have $h(1)=0$ for $t=(k-1)\tau+1$. This means that there is no gap between MFL and centralized MGD when local update is only performed once after the global aggregation.

It is easy to find that if $\tau=1$, $t-(k-1)\tau$ is either 0 or 1. Because $h(1)=h(0)=0$, the upper bound in \eqref{gap}
is zero, and there is no gap between $F(\textup{\textbf{w}}(t))$ and $F(\textup{\textbf{w}}_{[k]}(t))$ from \eqref{gapff}. This is consistent with Appendix \ref{A} where MFL yields centralized MGD for $\tau=1$. 
In any interval $[k]$, we have $t-(k-1)\tau\in[0, \tau]$. If $\tau>1$, $t-(k-1)\tau$ can be larger than 1. When $x>1$, we know that $h(x)$ increases with $x$. According to the definition of $A$, $B$, $E$ and $F$, we can obtain $\gamma A>1$, $\gamma B<1$ and $E,F>0$ easily.  
Because $0<\gamma<1$, the last term will linearly decrease with $x$ when $x$ is large. 
Therefore, the first exponential term $E(\gamma A)^x$ in \eqref{hx} will be dominant when $x$ is large and the gap between $\textbf{w}(t)$ and $\textbf{w}_{[k]}(t)$ increases exponentially with $t$. 

Also we find $h(x)$ is proportional to the average gradient gap $\delta$. 
It is because the greater the local gradient divergences at different nodes are, the larger the gap will be.
So considering the extreme situation where all nodes have the same data samples ($\delta=0$ because the local loss function are the same), the gap between $\textbf{w}(t)$ and $\textbf{w}(t)$ is zero and MFL is equivalent to centralized MGD.

\subsection{Global Convergence}
We have derived an upper bound between $F(\textup{\textbf{w}}(t))$ and $F(\textup{\textbf{w}}_{[k]}(t))$ for $t\in[k]$. According to the definition of MFL, in the beginning of each interval $[k]$, we set $\textbf{d}_{[k]}((k-1)\tau)=\textbf{d}((k-1)\tau)$ and $\textbf{w}_{[k]}((k-1)\tau)=\textbf{w}((k-1)\tau)$. The global  upper bound on the convergence rate of MFL can be derived based on Proposition \ref{thm2}.

The following definitions are made to facilitate analysis. Firstly, we use $\theta_{[k]}(t)$ to denote the angle between vector $\nabla F(\textbf{w}_{[k]}(t)$ and $\textbf{d}_{[k]}(t)$ for $t\in[k]$, i.e.,
$$\cos\theta_{[k]}(t)\triangleq\frac{\nabla F(\textbf{w}_{[k]}(t))^\mathrm{T}\textbf{d}_{[k]}(t)}{\|\nabla F(\textbf{w}_{[k]}(t)\|\|\textbf{d}_{[k]}(t)\|}$$
where $\theta$ is defined as the maximum value of $\theta_{[k]}(t)$ for $1\leq k\leq K$ with $t\in[k]$, i.e.,
$$\theta\triangleq\max_{1\leq k\leq K,t\in[k]}\theta_{[k]}(t).$$
Then we define 
$$p\triangleq\max_{1\leq k\leq K,t\in[k]}\frac{\|\textbf{d}_{[k]}(t)\|}{\|\nabla F(\textbf{w}_{[k]}(t))\|}$$
and $$\omega\triangleq\min_k \frac{1}{\|\textbf{w}((k-1)\tau)-\textbf{w}^*\|^2}.$$

Based on Proposition \ref{thm2} which gives an upper bound of loss function difference between MFL and centralized MGD, global convergence rate of MFL can be derived as follows.
\begin{lemma}\label{globalconvergence}
	If the following conditions are satisfied:\\
	\indent    1) $\cos\theta\geq0$, $0<\eta\beta<1$ and $0\leq\gamma<1$;\\
	\indent There   exists $\varepsilon>0$ which makes\\
	\indent    2) $F(\textup{\textbf{w}}_{[k]}(k\tau))-F(\textup{\textbf{w}}^*)\geq\varepsilon$ for all $k$;\\ 
	\indent    3) $F(\textup{\textbf{w}}(T))-F(\textup{\textbf{w}}^*)\geq\varepsilon$; \\
	\indent    4) $\omega\alpha-\frac{\rho h(\tau)}{\tau\varepsilon^2}>0$ hold,\\
	then we have 
	\begin{align}
	F(\textup{\textbf{w}}(T))-F(\textup{\textbf{w}}^*)\leq\frac{1}{T\left(\omega\alpha-\frac{\rho h(\tau)}{\tau\varepsilon^2}\right)}
	\end{align} where we defined $$\alpha\triangleq\!\eta(1\!-\!\frac{\beta\eta}{2})\!+\!\eta\gamma(1-\beta\eta)\cos\theta\!-\!\frac{\beta\eta^2\gamma^2 p^2}{2}.$$
\end{lemma}
\begin{proof}
	The proof is presented in Appendix \ref{D}.
\end{proof}

On the basis of Lemma \ref{globalconvergence}, we further derive the following proposition which demonstrates the global convergence of MFL and gives its  upper bound on convergence rate.
\begin{proposition}[MFL global convergence]\label{thm3}
	Given $\cos\theta\geq0$, $0<\eta\beta<1$, $0\leq\gamma<1$ and $\alpha>0$, we have
	\begin{align}\label{19}
	F(\textup{\textbf{w}}^\mathrm{f})-F(\textup{\textbf{w}}^*)\leq\frac{1}{2T\omega\alpha}+\sqrt{\frac{1}{4T^2\omega^2\alpha^2 }+\frac{\rho h(\tau)}{\omega\alpha\tau}}+\rho h(\tau).
	\end{align}
\end{proposition}
\begin{proof}
	The specific proving process is shown in Appendix \ref{E}.
\end{proof}

According to the above Proposition \ref{thm3}, we get an upper bound of $F(\textup{\textbf{w}}^\mathrm{f})-F(\textup{\textbf{w}}^*)$ which is a function of $T$ and $\tau$. From inequality \eqref{19}, we can find that MFL linearly converges to a lower bound $\sqrt{\frac{\rho h(\tau)}{\omega\alpha\tau}}+\rho h(\tau)$. Because $h(\tau)$ is related to $\tau$ and $\delta$, aggregation intervals ($\tau>1$) and different data distribution collectively lead to that MFL does not converge to the optimum. 

In the following, we discuss the influence of $\tau$ on the convergence bound. If $\tau=1$, we have $\rho h(\tau)=0$ so that $F(\textup{\textbf{w}}^\mathrm{f})-F(\textup{\textbf{w}}^*)$ linearly converges to zero as $T\to \infty$, and the convergence rate yields $\frac{1}{T\omega\alpha}$. Noting $h(\tau)>0$ if $\tau>1$, we can find that in this case, $F(\textup{\textbf{w}}^\mathrm{f})-F(\textup{\textbf{w}}^*)$ converges to a non-zero bound $\sqrt{\frac{\rho h(\tau)}{\omega\alpha\tau}}+\rho h(\tau)$ as $T\to \infty$. On the one hand, if there does not exist communication resources limit, setting aggregation frequency $\tau=1$ and performing global aggregation after each local update can reach the optimal convergence performance of MFL. On the other hand, aggregation interval ($\tau>1$) can let MFL effectively utilize the communication resources of each node, but bring about a decline of convergence performance.

\section{Comparison Between FL and MFL}
In this section, we make a  comparison of convergence performance between MFL and FL.

The closed-form solution of the upper bound on FL convergence rate has been derived in \citep[Theorem 2]{wang2019adaptive}. It is presented as follows.
\begin{align}\label{20}
F(\textbf{w}_{F\!L}^\mathrm{f})\!-\!F(\textup{\textbf{w}}^*)\!\leq\!\frac{1}{2\eta\varphi T}\!+\!\sqrt{\frac{1}{4\eta^2\varphi^2 T^2}\!+\!\frac{\rho h_{F\!L}(\tau)}{\eta\varphi\tau}}\!+\!\rho h_{F\!L}(\tau).
\end{align}
According to \cite{wang2019adaptive}, $$h_{F\!L}(\tau)=\frac{\delta}{\beta}((\eta\beta+1)^\tau-1)-\eta\delta\tau$$ and $\varphi=\omega_{F\!L}(1-\frac{\eta\beta}{2})$ where the expression of $\omega_{FL}$ is consistent with $\omega$. Differing from that of $\omega$, $\textbf{w}((k-1)\tau)$ in the definition of $\omega_{FL}$ is the global model parameter of FL.  


We assume that both MFL and FL solutions are applied in the system model proposed in Fig. \ref{fig2}. They are trained based on the same training dataset with the same machine learning model. The loss function $F_i(\cdot)$ and global loss function $F(\cdot)$ of MFL and FL are the same, respective. The corresponding parameters of MFL and FL are equivalent including $\tau$, $\eta$, $\rho$, $\delta$ and $\beta$.
We set the same initial value $\textbf{w}(0)$ of MFL and FL. Because both MFL and FL are  convergent, we have 
$\omega=\frac{1}{\|\textbf{w}(0)-\textbf{w}^*\|^2}$. Then according to the definitions of $\omega$ and $\omega_{F\!L}$, we have $w=w_{F\!L}$. 
Therefore, the corresponding parameters of MFL and FL are the same and we can compare the convergences between FL and MFL conveniently.

%

For convenience, we use $f_1(T)$ and $f_2(T)$ to denote the  upper bound on convergence rate of \textup{MFL} and \textup{FL}, respectively. Then we have
\begin{align}\label{f1}
f_1(T)\triangleq\frac{1}{2T\omega\alpha}+\sqrt{\frac{1}{4T^2\omega^2\alpha^2 }+\frac{\rho h(\tau)}{\omega\alpha\tau}}+\rho h(\tau)
\end{align}
and
\begin{align}\label{f2}
f_2(T)\triangleq\frac{1}{2\eta\varphi T}\!+\!\sqrt{\frac{1}{4\eta^2\varphi^2 T^2}\!+\!\frac{\rho h_{F\!L}(\tau)}{\eta\varphi\tau}}\!+\!\rho h_{F\!L}(\tau).
\end{align}
We consider the special case of $\gamma\to0$. For $\omega\alpha$ and $\eta\varphi$, we can obtain $\omega\alpha\to\omega\eta(1-\frac{\beta\eta}{2})=\eta\varphi$ from the definition of $\alpha$. Then for $h(\tau)$ and $h_{FL}(\tau)$, we have $\gamma A\to \eta\beta+1$ and $\gamma B\to0$. Because $\frac{A}{A-B}\to1$ and $\frac{B}{A-B}\to0$, we can further get $E\to\frac{1}{\eta\beta}$ and $F\to0$ from the definitions of $E$ and $F$. So, according to \eqref{hx}, we have 
\begin{align*}
\lim_{\gamma\to0} h(\tau)&=\eta\delta\left[ \frac{1}{\eta\beta}(\eta\beta+1)^\tau\!-\!\frac{1}{\eta\beta}\!-\tau\right]\\
&=\frac{\delta}{\beta}((1+\eta\beta)^\tau-1)-\eta\delta \tau=h_{FL}(\tau).
\end{align*}
Hence, by the above analysis under $\gamma\to0$, we can find MFL and FL have the same upper bound on convergence rate. This fact is consistent with the property that if $\gamma=0$, MFL degenerates into FL and has the same convergence rate with FL.

To avoid complicated calculations over the expressions of $f_1(T)$ and $f_2(T)$, we consider the special case of small $\eta$ ($\eta\to0$) which is typically used in simulations.
\begin{lemma}\label{lem}
	If $\eta\to0$, existing $T_1\geq1$, we have $\frac{1}{2T\omega\alpha}$ dominates in $f_1(T)$ and $\frac{1}{2\eta\varphi T}$ dominates in $f_2(T)$ for $T<T_1$,  
	i.e., 
	$$\frac{1}{2T\omega\alpha}\gg\max\left\{\rho h(\tau),\sqrt{\frac{\rho h(\tau)}{\omega\alpha\tau}}\right\}$$
	and
	$$\frac{1}{2\eta\varphi T}\gg\max\left\{\rho h_{FL}(\tau),\sqrt{\frac{\rho h_{FL}(\tau)}{\eta\varphi\tau}}\right\}.$$
	Then we have 
	$$f_1(T)\approx\frac{1}{T\omega\alpha}$$
	and
	$$f_2(T)\approx\frac{1}{T\eta\varphi}$$ for $T<T_1$.
\end{lemma}
\begin{proof}
	Considering \eqref{f1},	if $\eta\to0$,  we have $\alpha\to 0$ from the definition of $\alpha$  and $h(\tau)\to 0$ from Appendix \ref{hospital}. So we can easily derive 
	$\omega\alpha\rho h(\tau)\to 0$ and $\sqrt{\frac{\omega\alpha\rho h(\tau)}{\tau}}\to 0$. Then we can find $T_1\geq1$ which satisfies $\frac{1}{2T}\gg\omega\alpha\rho h(\tau)$ and $\frac{1}{2T}\gg\sqrt{\frac{\omega\alpha\rho h(\tau)}{\tau}}$ for $T<T_1$. Hence, we have  $\frac{1}{2T\omega\alpha}$ dominates in $f_1(T)$ and $f_1(T)\approx\frac{1}{T\omega\alpha}$. 
	
	For the same reason, considering \eqref{f2}, if $\eta\to0$, we have $\eta\varphi\to 0$ from the definition of $\varphi$ and $h_{FL}(\tau)\to 0$ from its definition. So we can easily derive 
	$\eta\varphi\rho h_{FL}(\tau)\to 0$ and $\sqrt{\frac{\eta\varphi\rho h_{FL}(\tau)}{\tau}}\to 0$. Then for $T<T_1$, $\frac{1}{2T}\gg\eta\varphi\rho h_{FL}(\tau)$ and $\frac{1}{2T}\gg\sqrt{\frac{\eta\varphi\rho h_{FL}(\tau)}{\tau}}$. Hence, we have  $\frac{1}{2\eta\varphi T}$ dominates in $f_2(T)$ and $f_2(T)\approx\frac{1}{T\eta\varphi}$. 
\end{proof}
From Lemma \ref{lem}, we can find that if the learning step size is small and the local update index is not too large, the influence of the update intervals on convergence rate can be ignored. So taking this property into account, we can set a small learning step size to cut down communication resources consumption by increasing computation resources cost.

\begin{proposition}[Accelerated convergence of MFL]\label{pro}
	If the following conditions are satisfied:\\
	\indent	1) $\eta\to0$;\\
	\indent	2) $T<T_1$;\\
	\indent	3) $0<\gamma<1$,\\
	\textup{MFL} converges faster than \textup{FL}, i.e.,
	$$f_1(T)<f_2(T).$$ 
\end{proposition}
\begin{proof}
	From condition 1 and condition 2, we have $f_1(T)\approx\frac{1}{T\omega\alpha}$ and $f_2(T)\approx\frac{1}{T\eta\varphi}$.  Due to the definition of $\alpha$ and $\varphi$, inequality $0<\gamma<\frac{2(1-\beta\eta)\cos\theta}{\beta \eta p^2}$ is equivalent to $\omega\alpha>\eta\varphi$. So if $\omega\alpha>\eta\varphi$,
	it is obvious that $\frac{1}{T\omega\alpha}<\frac{1}{T\eta\varphi}$, i.e.,
	$f_1(T)<f_2(T)$. From condition 1,  $\frac{2(1-\beta\eta)\cos\theta}{\beta \eta p^2}$ is larger than 1 definitely. So  condition 3 is the range of MFL convergence acceleration after combining with MGD convergence guarantee $0<\gamma<1$.

\end{proof}


\begin{figure*}[!t]
	\centering
	\subfigure[SVM]{\includegraphics[scale=0.44]{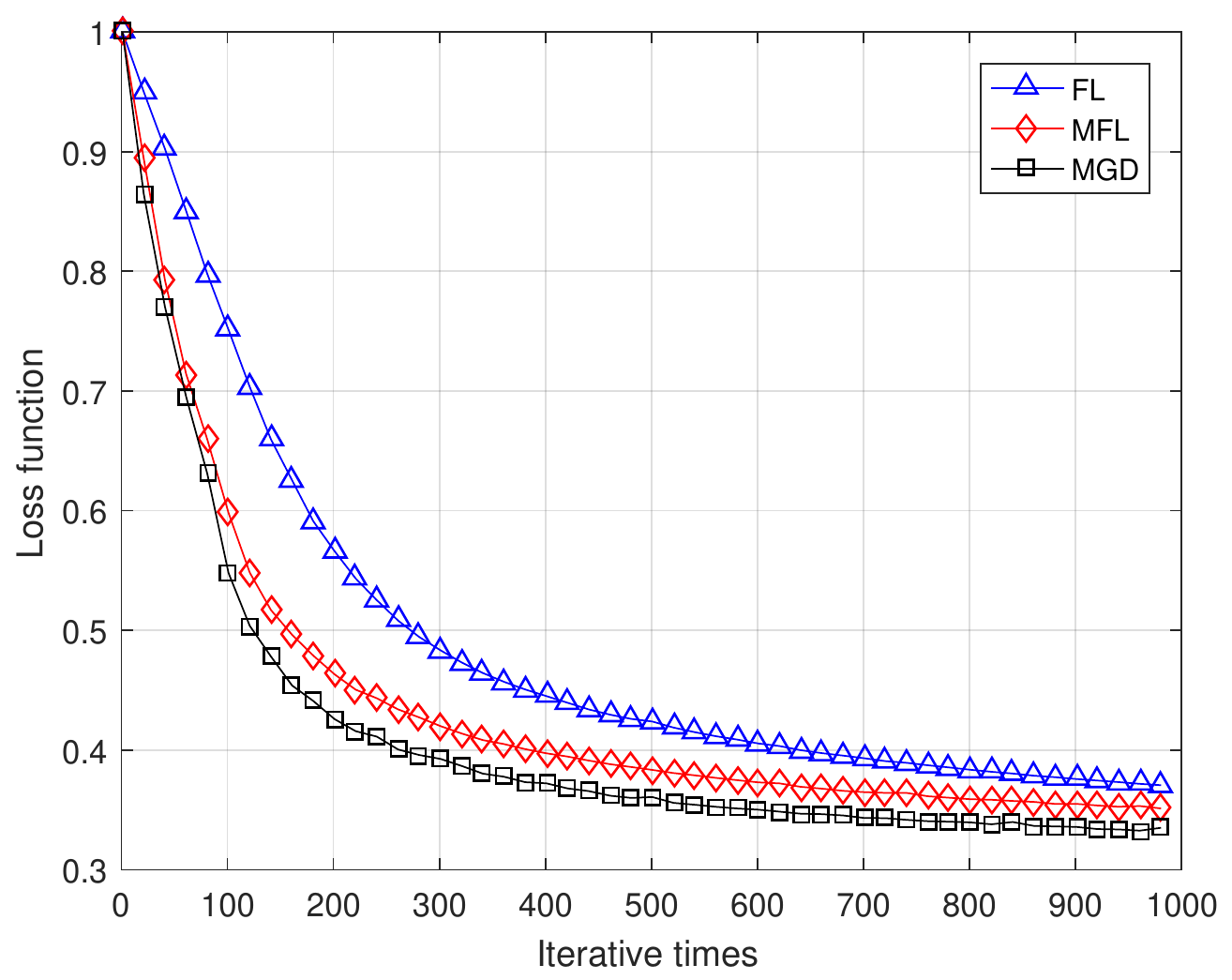}}
	\subfigure[SVM]{\includegraphics[scale=0.44]{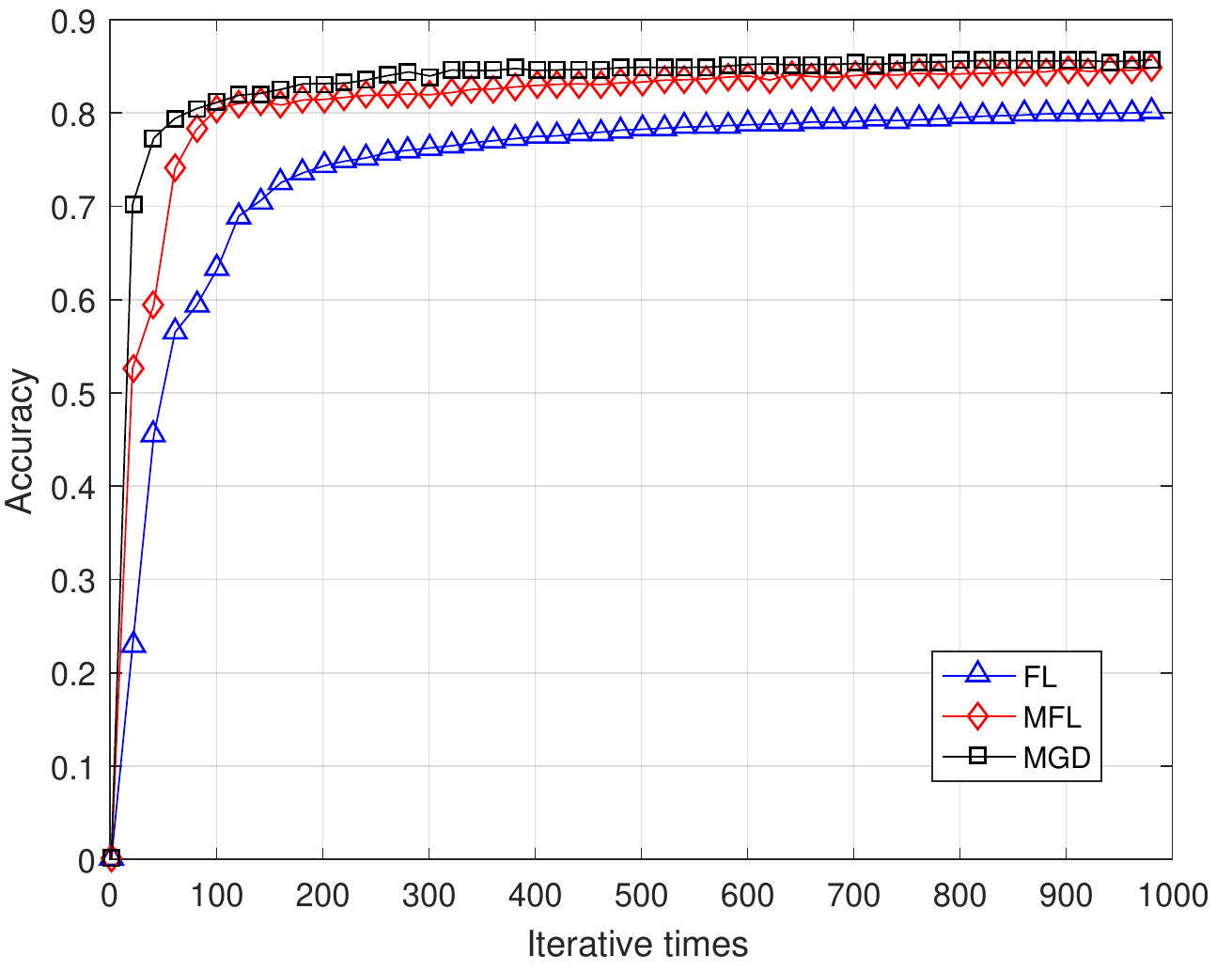}}
	
	\subfigure[Linear regression]{\includegraphics[scale=0.44]{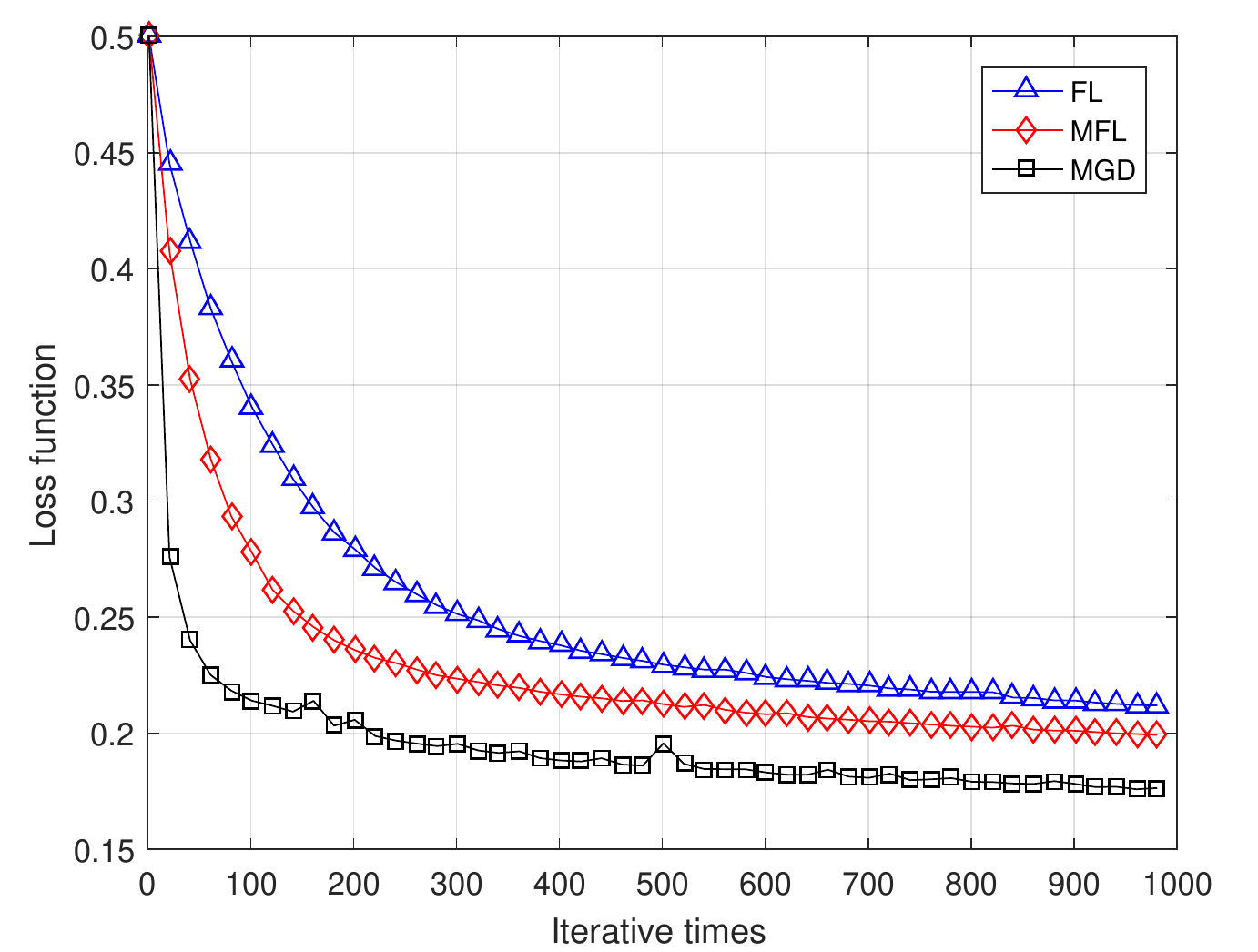}}
	\subfigure[Logistic regression]{\includegraphics[scale=0.44]{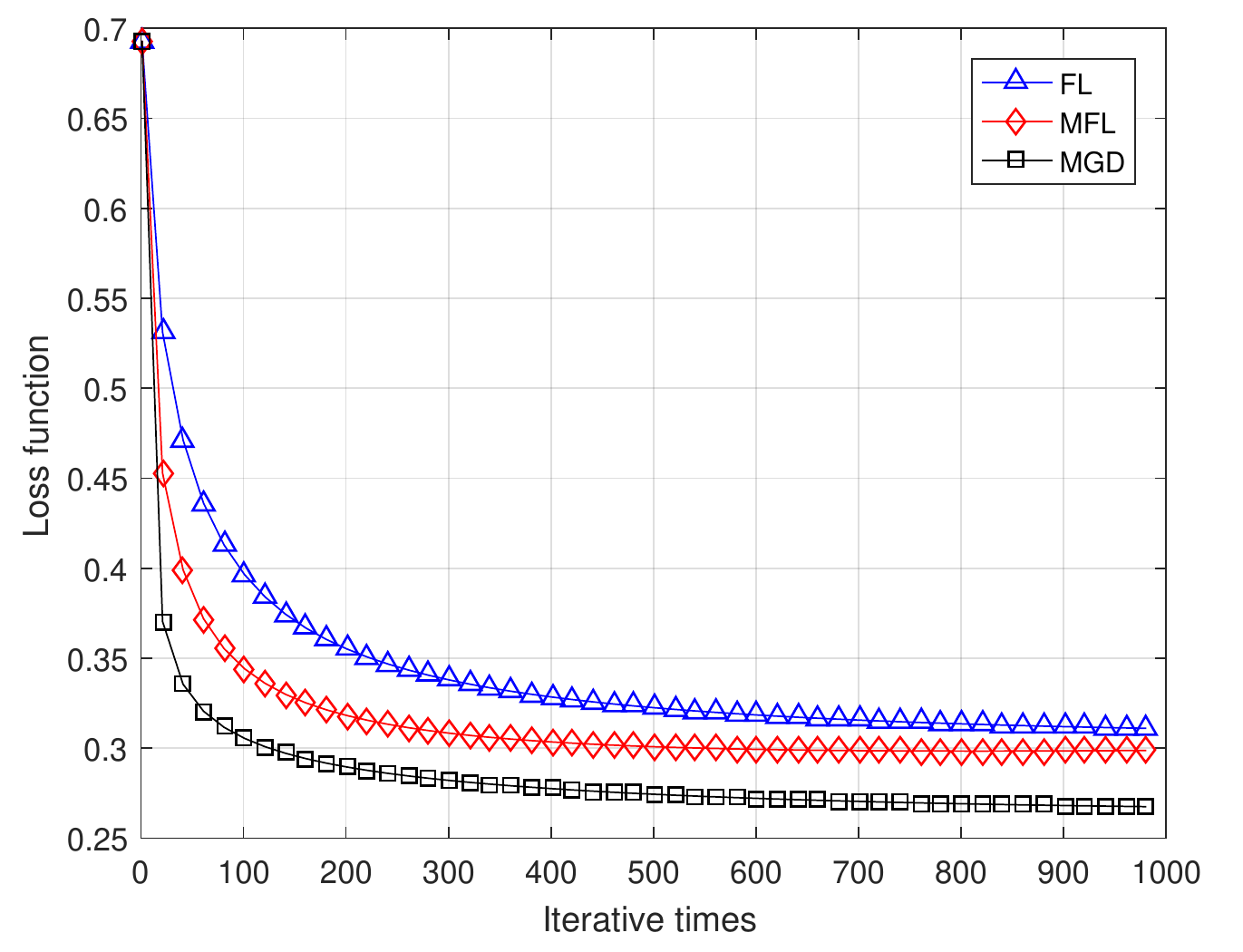}}
	\caption{Loss function values and testing accuracy under FL, MFL and MGD.
		(a) and (b) are the loss function and test accuracy curve of SVM, respectively;
		(c) and (d) are the loss function curves of linear regression and logistic regression, respectively.
	}
	\label{Convergence}
\end{figure*}

\section{Simulation and discussion}
In this section, we build and evaluate MFL system based on MNIST dataset. We first describe the simulation environment and the relevant sets of parameters. Then, we present and evaluate the comparative simulation results of MFL, FL and MGD under different machine learning models. Finally, the extensive experiments are implemented to explore the impact of $\gamma$ and $\tau$ on MFL convergence performance.

\subsection{Simulation Setup}

Using the Python, we build a federated network framework where distributed edge nodes coordinate with the central server. In our network, the number of edge nodes can be chosen arbitrarily. SVM, linear regression and logistic regression can be chosen for model training. Their loss functions at node $i$ are presented as in Table \ref{table_time} \cite{shalev2014understanding}. Note that $|\mathcal{D}_i|$ is the number of training samples in node $i$ and the loss function of logistic regression is cross-entropy. For logistic regression, model output $\sigma(\textbf{w}, \textbf{x}_j)$ is sigmoid function for non-linear transform. It is defined by
\begin{align}\label{sigma}
\sigma(\textbf{w},\textbf{x}_j)\triangleq\frac{1}{1+e^{-\textbf{w}^\mathrm{T}\textbf{x}_j}}.
\end{align}

\renewcommand\arraystretch{1.5}
\begin{table}[!t]  
	\caption{Loss function of three machine learning models}
	\centering
	\label{table_time}
	\renewcommand\arraystretch{1.5}
	\begin{tabular}{cp{5.8cm}}  
		
		\toprule[1pt]   
		
		\multicolumn{1}{c}{	\textbf{Model}} & \multicolumn{1}{c}{\textbf{Loss function}}  \\  
		
		\midrule 
		
		\multicolumn{1}{c}{SVM} &  $\frac{\lambda}{2}\|\textbf{w}\|^2+\frac{1}{2|\mathcal{D}_i|}\sum_j\max\{0; 1-y_j\textbf{w}^{\mathrm{T}}\textbf{x}_j\}$   \\  
		
		\multicolumn{1}{c}{Linear regression} &  $\frac{1}{2|\mathcal{D}_i|}\sum_j\|y_j-\textbf{w}^\mathrm{T}\textbf{x}_j\|^2$    \\    
		
		\multicolumn{1}{c}{Logistic regression} & $-\!\frac{1}{|\mathcal{D}_i|}\!\sum_j\!\|\!y_j\!\log \sigma(\textbf{w},\textbf{x}_j)\!+\!(1\!-\!y_j)\log(1\!-\!\sigma(\textbf{w},\textbf{x}_j))\|$ where $\sigma(\textbf{w},\textbf{x}_j)$ is given as \eqref{sigma}\\
		
		\bottomrule[1pt]    
		
	\end{tabular}
	
\end{table}
\begin{figure*}[!t]
	\centering
	\subfigure[]{\includegraphics[scale=0.42]{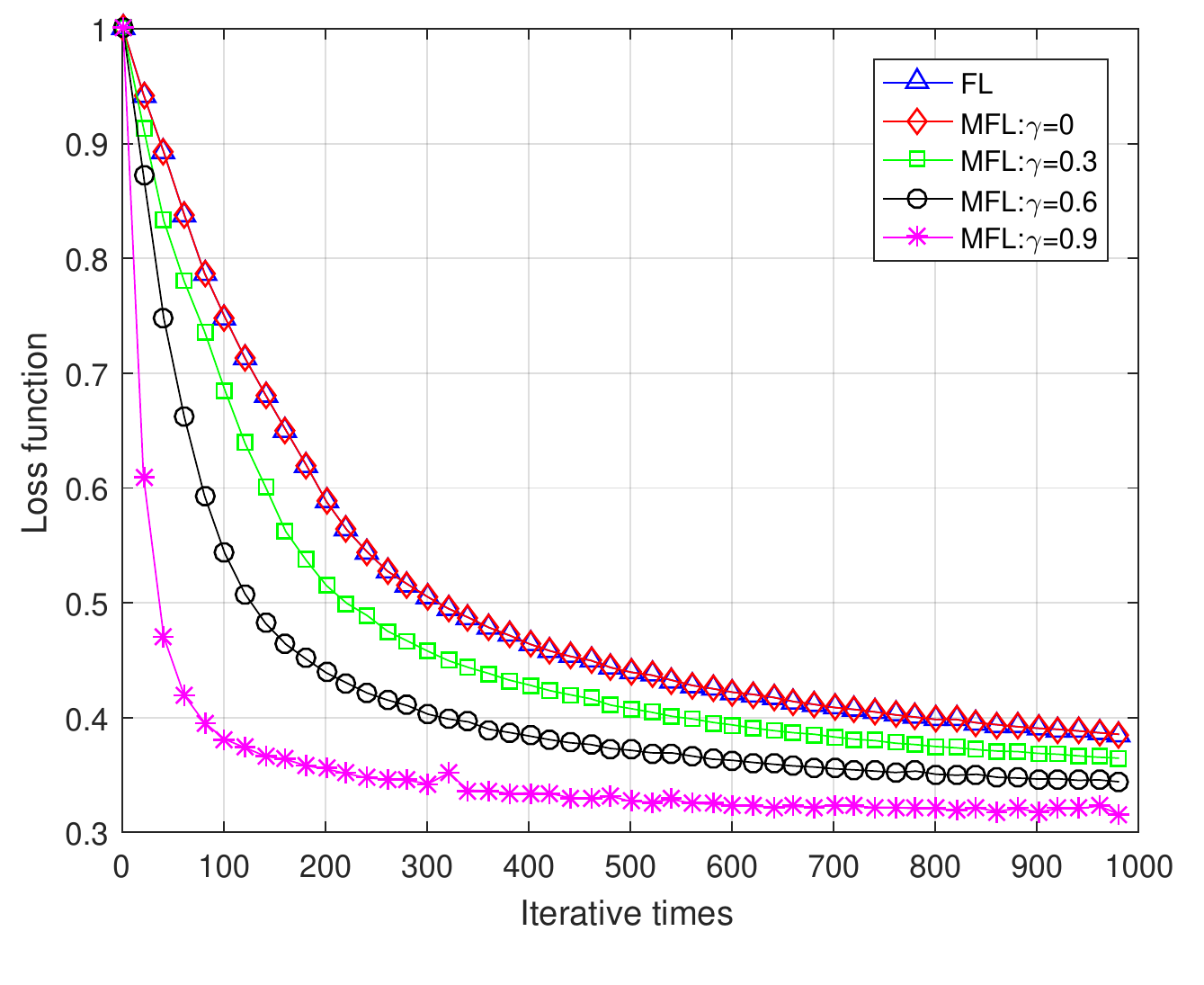}}
	\subfigure[]{\includegraphics[scale=0.42]{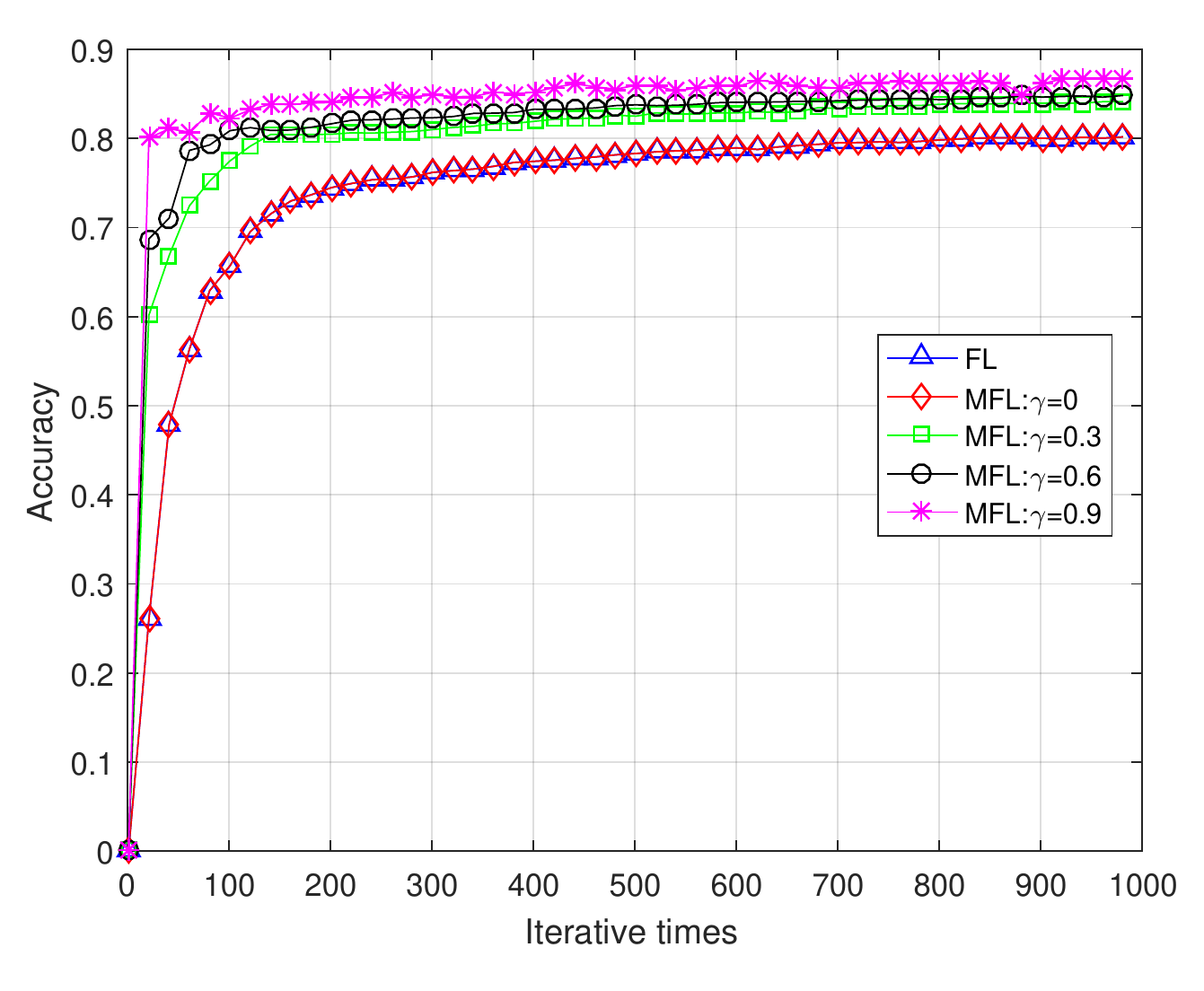}}
	\subfigure[]{\includegraphics[scale=0.42]{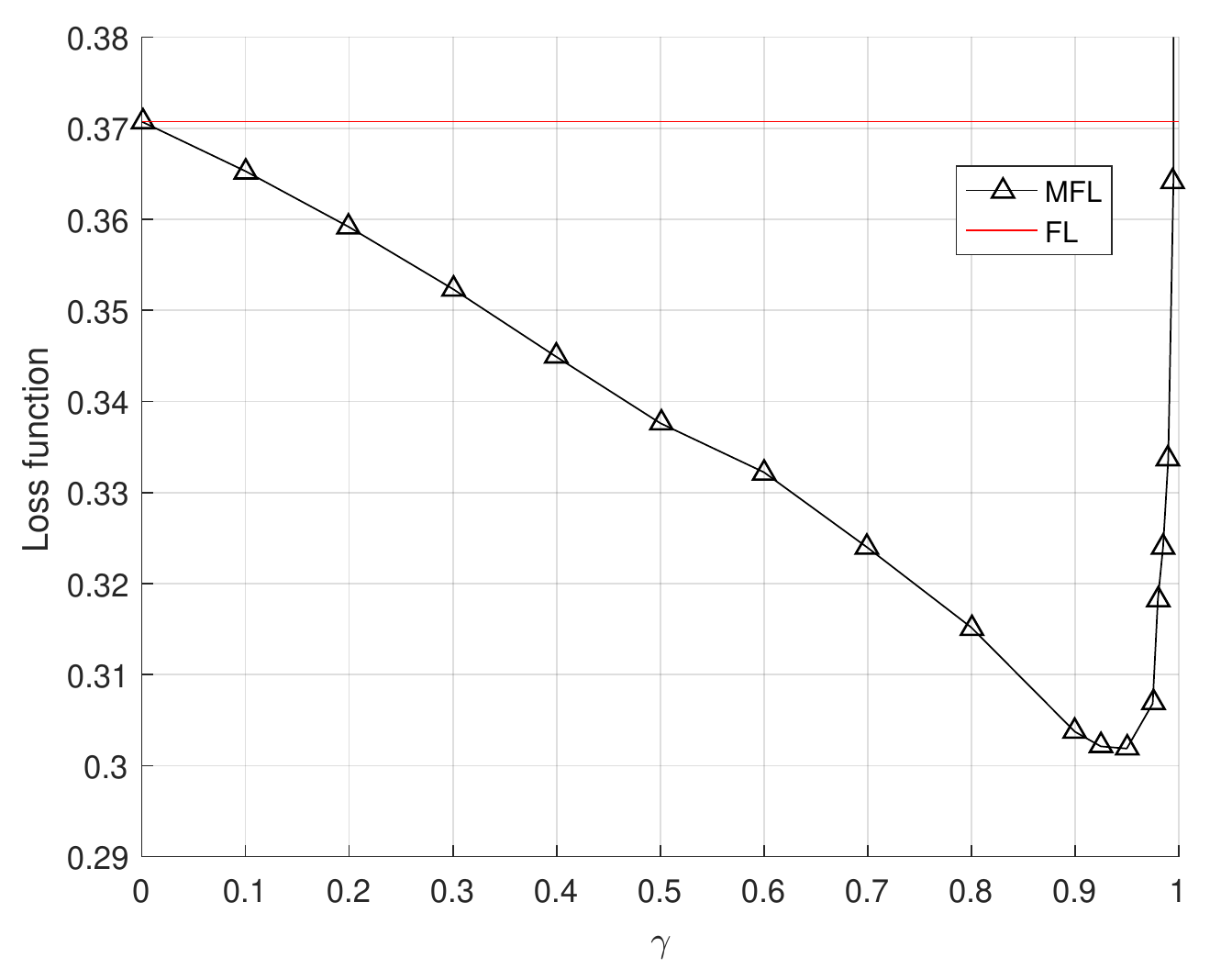}}
	\caption{The influence of $\gamma$ on MFL convergence.
		(a)Loss function values with iterative times under different $\gamma$;
		(b)Testing accuracy with iterative times under different $\gamma$;
		(c)Loss function values with $\gamma$ when $T=1000$.}
	\label{gamma}
\end{figure*}

These three models are trained and tested on MNIST dataset which contains 50,000 training handwritten digits and 10,000 testing handwritten digits. In our experiments, we only utilize 5,000 training samples and 5,000 testing samples because of the limited processing capacities of GD and MGD. 
In this dataset, the $j$-th sample $\textbf{x}_j$ is a 784-dimensional input vector which is vectorized from $28\times28$ pixel matrix and $y_j$ is the scalar label  corresponding to $\textbf{x}_j$. SVM, linear and logistic regression are used to classify whether the digit is even or odd. 

In our experimentation, training and testing samples are randomly allocated to each node, which means the information of each node is uniform. We also simulate FL and centralized MGD as benchmarks. For experimental sets, the system model is set to 4 edge nodes. If the image of $\textbf{x}_j$ represents an even number, then we set $y_j=1$. Otherwise, $y_j=-1$. But for logistic regression, we set $y_j=1$ for the even number and $y_j=0$ for the odd.
The same initializations of model parameters are performed and the same data distributions are set for MFL and FL. Also $\textbf{d}_i(0)=\textbf{0}$ is set for node $i$. We set the  learning step size $\eta=0.002$ which is sufficiently small, SVM parameter $\lambda=0.3$ and the total number of local iterations $T=1,000$ for the following simulations.

\begin{figure}[!t]	
	
	\centering
	\includegraphics[scale=0.6]{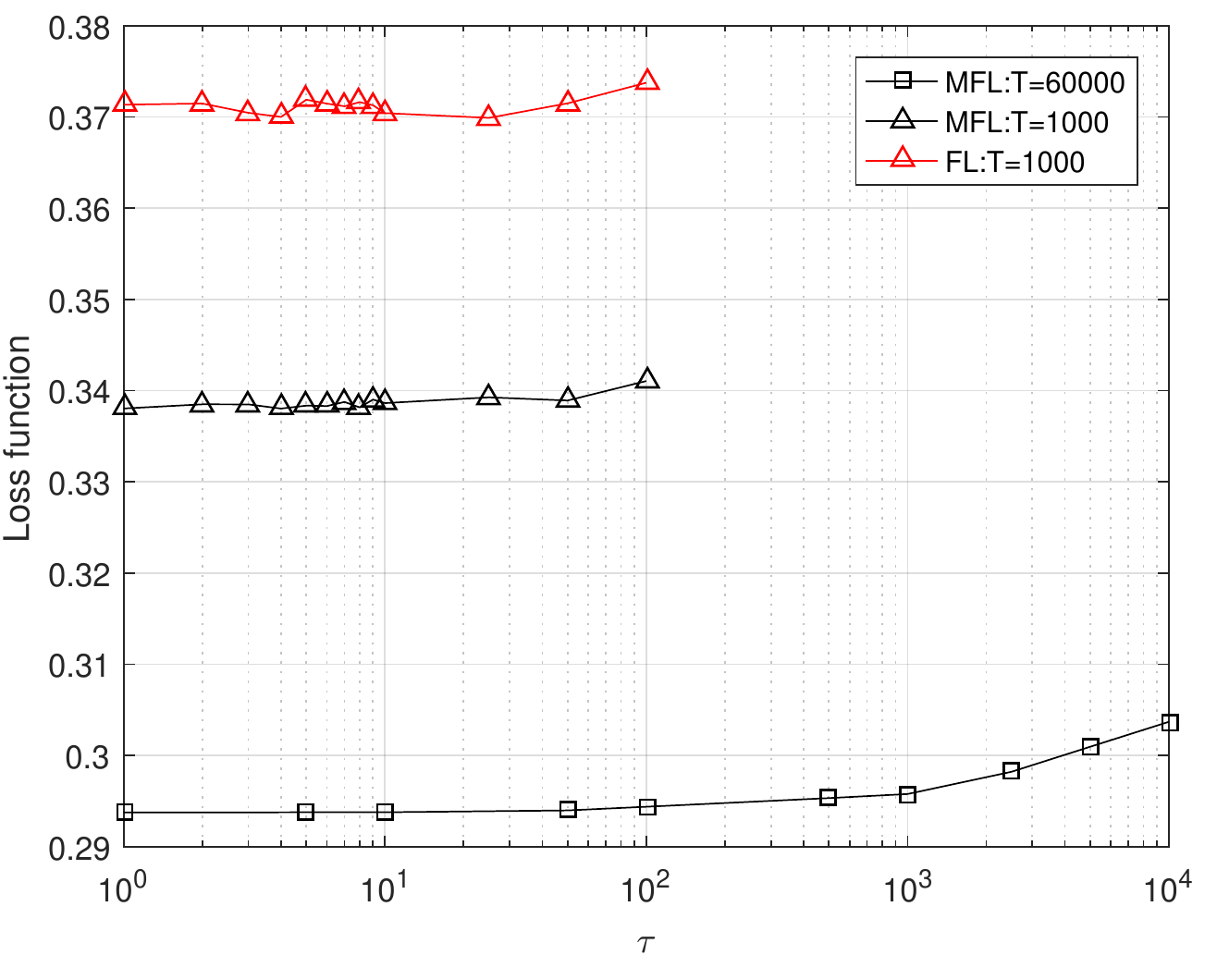}        
	\caption{Loss function values with $\tau$}
	\label{tau}	
	
\end{figure}

\subsection{Simulation Evaluation}
In this subsection, we verify the convergence acceleration of MFL and explore the effects of $\gamma$ and $\tau$ on MFL convergence by simulation evaluations.

\subsubsection{Convergence}
In our first simulation, the models of SVM, linear regression and logistic regression are trained and we verify the convergence of MFL. We set aggregation frequency $\tau=4$ and momentum attenuation factor $\gamma=0.5$. MFL, FL and MGD are performed based on the three machine learning model. 
MGD is implemented based on the global dataset which is obtained by gathering the distributed data on all nodes. The global loss functions of this three solutions are defined based on the same global training and testing data.  

The curves of loss function values and accuracy  with iterative times are presented in Fig. \ref{Convergence}. We can see that the loss function curves for all the learning models are gradually convergent with iterative times. Similarly, the test accuracy curves for SVM  gradually rise until convergence with iterative times.  Therefore, convergence of MFL is verified. 
We also see that the descent speeds of MFL loss function curves on three learning models are always faster than those of FL while the centralized  MGD convergence speeds are fastest. So compared with FL, MFL provides a significant improvement on convergence rate. MGD converges with the fastest speed because MFL and FL suffer the delay in global gradient update for $\tau=4$.

Because linear and logistic regression can not provide the testing accuracy curves, we focus on the SVM model in the following experiments and further explore the impact of MFL parameters on convergence rate.

\subsubsection{Effect of $\gamma$}

We evaluate the impact of $\gamma$ on the convergence rate of loss function. In this simulation, we still set aggregation frequency $\tau=4$. 

The experimental results are shown in Fig. \ref{gamma}. Subfigure (a) and (b) show that how different values of $\gamma$ affect the convergence curves of loss function and testing accuracy, respectively.
We can see that if $\gamma=0$, the loss function and accuracy curves of  MFL  overlap with the corresponding ones of FL because MFL is equivalent to FL for $\gamma=0$. When $\gamma$ increases from 0 to 0.9, we can see the convergence rates on both loss function curves and accuracy curves also gradually increase.
Subfigure (c) shows the change of final loss function value ($T=1000$) with $0<\gamma<1$. From this subfigure, we can find the final loss function values of MFL are always smaller than FL with $0<\gamma<1$. Compared with FL, convergence performance of MFL is improved. This is because $\frac{2(1-\beta\eta)\cos\theta}{\beta \eta p^2}>1$ and according to Proposition \ref{pro}, the accelerated convergence range of MFL is $0<\gamma<1$.
We can see that when $0<\gamma<0.95$, the loss function values decrease monotonically  with $\gamma$ so the convergence rate of MFL increases with $\gamma$. While $\gamma>0.95$, the loss function values of MFL start to increase with a gradual deterioration of MFL convergence performance, and in this situation, MFL can not remain convergence. If the $\gamma$ values are chosen to be close to 1, best around 0.9, MFL reaches the optimal convergence rate.


\subsubsection{Effect of $\tau$}
Finally, we evaluate the effect of different $\tau$ on loss function of MFL. 
We record the final loss function values with $\tau$ based on the three cases of $T=1,000$ for FL, $T=1,000$ for MFL and $T=60,000$ for MFL. We set $\gamma=0.5$ for MFL.
The curves for the three cases are presented in Fig. \ref{tau}. 
Comparing FL with MFL for $T=1,000$, we see that the final loss function values of MFL are smaller than those of FL for any $\tau$. As declared in Proposition \ref{pro}, under a small magnitude of $T$ and $\eta=0.002$ which is close to 0, MFL always converges much faster than FL. Further, for $T=1,000$, the effect of  $\tau$ on convergence is slight because the curves of FL and MFL are relatively plain. This can be explained by Lemma \ref{lem}, where $\frac{1}{2\eta\varphi T}$ and $\frac{1}{2\omega\alpha T}$ dominate the convergence upper-bound when the magnitude of $T$ is small. While $T=60,000$, change of $\tau$  affects convergence significantly and the final loss function values gradually increase with $\tau$. As the cases of $T=1,000$ for MFL and FL, the case of $T=60,000$ for MFL has a slight effect on convergence if $\tau<100$. But if $\tau>100$, MFL convergence performance is getting worse with $\tau$. According to the above analysis of $\tau$, setting an appropriate aggregation frequency will reduce convergence performance slightly with a decline of communication cost (in our cases, $\tau=100$).

\section{Conclusion}
In this paper, we have proposed MFL which performs MGD in local update step to solve the distributed machine learning problem. Firstly,
we have established global convergence properties of MFL and derived an upper bound on MFL convergence rate. 
This theoretical upper bound shows 
that the sequence generated by MFL linearly converges to the global optimum point under certain conditions. Then, compared with FL, MFL provides accelerated convergence performance under the given conditions as presented in Proposition \ref{pro}. Finally, based on MNIST dataset, our simulation results have verified the MFL convergence  and confirmed the accelerated convergence of MFL. 

\small
\bibliographystyle{unsrt} 
\bibliography{conference_041818} 
\normalsize
\appendices
\section{MFL vs. Centralized MGD}\label{A}
\begin{proposition}\label{thm1}
	When $\tau=1$, MFL is equivalent to centralized MGD, i.e., the update rules of MFL yield the following equations:
	\begin{align*}
	\textup{\textbf{d}}(t)&=\gamma \textup{\textbf{d}}(t-1)+\nabla F(\textup{\textbf{w}}(t-1))\\
	\textup{\textbf{w}}(t)&=\textup{\textbf{w}}(t-1)-\eta \textup{\textbf{d}}(t)	
	\end{align*}
\end{proposition} 

\begin{proof}
	When $\tau=1$, according to the definition of MFL, we have $\widetilde{\textbf{d}}_i(t)=\textbf{d}(t)$ and $\widetilde{\textbf{w}}_i(t)=\textbf{w}(t)$ for node $i$. Then
	\begin{align*}
	\textbf{w}(t)&=\frac{\sum_{i=1}^{N} |\mathcal{D}_i| \textbf{w}_i(t)}{|\mathcal{D}|}\\
	&=\frac{\sum_{i=1}^{N} |\mathcal{D}_i| (\widetilde{\textbf{w}}_i(t-1)-\eta\textbf{d}_i(t))}{|\mathcal{D}|}\\
	&=\textbf{w}(t-1)-\eta \frac{\sum_{i=1}^{N} |\mathcal{D}_i| (\gamma \widetilde{\textbf{d}}_i(t-1)+\nabla F_i(\widetilde{\textbf{w}}_i(t-1))}{|\mathcal{D}|}\\
	&=\textbf{w}(t-1)-\eta \gamma \textbf{d}(t-1) - \eta \frac{\sum_{i=1}^{N} |\mathcal{D}_i| \nabla F_i(\textbf{w}(t-1))}{|\mathcal{D}|}\\
	&=\textbf{w}(t-1)-\eta \gamma \textbf{d}(t-1) -\eta \nabla F(\textbf{w}(t-1))
	\end{align*}
	where the last equality is derived because
	\begin{align*}
	\frac{\sum_{i=1}^{N} |\mathcal{D}_i| \nabla F_i(\textbf{w})}{|\mathcal{D}|}=\nabla \left(\frac{\sum_{i=1}^{N} |\mathcal{D}_i| F_i(\textbf{w})}{|\mathcal{D}|}\right)=\nabla F(\textbf{w})
	\end{align*}
	\smallskip
	from the linearity of the gradient operator. So we can get 
	\begin{align*}
	\textbf{d}(t)&=\gamma \textbf{d}(t-1)-\nabla F(\textbf{w}(t-1))\\
	\textbf{w}(t)&=\textbf{w}(t-1)-\eta\textbf{d}(t)
	\end{align*}
	finally.
\end{proof}

\section{Proof of Proposition 1}\label{B}
\textbf{We first introduce a lemma about sequence for the following proof preparation. The following lemma \ref{lemma2} provides the general formulas which can be used to match the inequality in Lemma \ref{lemma3}.}
\begin{lemma}\label{lemma2}
	\textit{Given
		\begin{align}
		&a_t=\frac{\delta_i}{\beta}\left (\frac{\frac{1+\eta \beta }{\gamma }-B}{A-B}A^{t}-\frac{\frac{1+\eta \beta }{\gamma }-A}{A-B}B^{t}\right )\\
		&\label{a+b}A+B=\frac{1+\gamma+\eta\beta}{\gamma}\\
		&\label{ab}AB=\frac{1}{\gamma}
		\end{align}
		where $t=0,1,2,..., 0<\gamma<1, \eta\beta>0$, we have
		\begin{align}
		a_{t-1}+\eta\beta\sum_{i=0}^{t-1}a_i=\gamma a_{t}.
		\end{align}}
\end{lemma} 
\begin{proof}
	For convenience, we define $C\triangleq\frac{\frac{1+\eta \beta }{\gamma }-B}{A-B}=\frac{A-1}{A-B}$ and
	$D\triangleq\frac{A-\frac{1+\eta \beta }{\gamma }}{A-B}=\frac{1-B}{A-B}$. So $a_t=\frac{\delta_i}{\beta}(CA^t+DB^t)$.
	After eliminating B or A by \eqref{a+b} and \eqref{ab}, we get
	\begin{align}\label{a.b}
	&\gamma x^2-(1+\gamma+\eta\beta)x+1=0.
	\end{align} 
	Because $\eta\beta>0$, we always have $\Delta=(1+\gamma+\eta\beta)^2-4\gamma>0$. Thus $A$ and $B$ can be derived as follows:
	\begin{align}
	\label{defA}A&=\frac{(1\!+\!\gamma+\!\eta\beta)\!+\!\sqrt{(1\!+\!\gamma\!+\!\eta\beta)^2\!-\!4\gamma}}{2\gamma}\\
	\label{defB}B&=\frac{(1\!+\!\gamma+\!\eta\beta)\!-\!\sqrt{(1\!+\!\gamma\!+\!\eta\beta)^2\!-\!4\gamma}}{2\gamma}.
	\end{align}
	Then we have 
	\begin{align*}
	&a_{t-1}+\eta\beta\sum_{i=0}^{t-1}a_i-\gamma a_{t}\\
	=&\frac{\delta_i}{\beta}\left (CA^{t-1}+DB^{t-1}\right )+\eta\beta\frac{\delta_i}{\beta}C\frac{A^t-1}{A-1}\\
	&+\eta\beta\frac{\delta_i}{\beta}D\frac{B^t-1}{B-1}-\gamma\frac{\delta_i}{\beta}CA^{t}-\gamma\frac{\delta_i}{\beta}DB^{t}\\
	=&\frac{\delta_i}{\beta}\left [\frac{A^{t-1}C}{1-A}(\gamma A^2-(1+\gamma+\eta\beta)A+1)\right.\\
	&\left.+\frac{B^{t-1}D}{1-B}(\gamma B^2\!-\!(1\!+\!\gamma\!+\!\eta\beta)B\!+\!1)\!-\!\eta\beta\left (  \frac{C}{A-1}\!+\!\frac{D}{B-1}\right )\right ]\\
	=&-\eta\delta_i\left(\frac{C}{A-1}+\frac{D}{B-1}\right)\\
	&\rightline{\text{(because $A$, $B$ satisfy \eqref{a.b})}}\\
	=&0.
	\end{align*}
	So Lemma \ref{lemma2} has been proven.
\end{proof}
\textbf{Then we study the difference between $\widetilde{\textbf{w}}_i(t)$ and $\textbf{w}_{[k]}(t)$ for node $i$. We obtain the following lemma.}
\begin{lemma}\label{lemma3}
	For any interval $[k]$ with $t\in[(k-1)\tau,k\tau)$, we have
	\begin{align*}
	\|\widetilde{\textup{\textbf{w}}}_{i}(t)-\textup{\textbf{w}}_{[k]}(t)\|\leq f_{i}(t-(k-1)\tau)
	\end{align*}
	where we define the function $f_i(t)$ as
	\begin{align*}
	f_{i}(t)\triangleq\frac{\delta_i}{\beta}\left ( \gamma ^{t} ( CA^{t}+DB^{t} \right ) -1).
	\end{align*}
\end{lemma} 
\begin{proof}
	
	When $t=(k-1)\tau$, we know that $$\widetilde{\textbf{w}}_{i}((k-1)\tau)=\textbf{w}((k-1)\tau)=\textbf{w}_{[k]}((k-1)\tau)$$ by the definition of $\textbf{w}_{[k]}(t)$, and we get $\|\widetilde{\textbf{w}}_{i}(t)-\textbf{w}_{[k]}(t)\|=0$. While $t=0$, we can get $f_i(0)=0$ easily. Hence, when $t=(k-1)\tau$, Lemma \ref{lemma3} holds.
	
	
	\textbf{Firstly, we derive the upper bound of $\|\widetilde{\textbf{d}}_i(t)-\textbf{d}_{[k]}(t)\|$ }as follows. For $t\in((k-1)\tau,k\tau]$,
	\begin{align}\label{gapd}
	&\|\widetilde{\textbf{d}}_i(t)-\textbf{d}_{[k]}(t)\|\notag\\
	&\!=\!\|\!\gamma \!\widetilde{\textbf{d}}_i(t\!-\!1)\!+\!\nabla \!F_i(\widetilde{\textbf{w}}_{i}(t\!-\!1))\notag\!-\!(\!\gamma \!\textbf{d}_{[k]}(t\!-\!1)+\nabla F(\textbf{w}_{[k]}(t\!-\!1\!)\!)\!)\!\|\notag \\
	&\rightline{\text{(from (6) and (13))}}\notag\\
	&\leq\!\|\nabla F_i(\widetilde{\textbf{w}}_{i}(t\!-\!1))\!-\!\nabla F_i(\textbf{w}_{[k]}(t\!-\!1))\notag\\
	&\!+\!\nabla F_i(\textbf{w}_{[k]}(t\!-\!1))\!-\!\nabla F(\textbf{w}_{[k]}(t\!-\!1))\|
	\!+\!\gamma\!\|\!\widetilde{\textbf{d}}_i(t\!-\!1)\!-\!\textbf{d}_{[k]}(t\!-\!1)\!\|\notag\\
	&\rightline{\text{(from triangle inequality and adding a zero term)}}\notag\\
	&\leq\beta\|\widetilde{\textbf{w}}_{i}(\!t\!-\!1)-\textbf{w}_{[k]}(t\!-\!1)\|\!+\!\delta_i+\gamma\|\widetilde{\textbf{d}}_i(t-1)-\textbf{d}_{[k]}(t-1)\|.\\
	&\rightline{\text{(from the $\beta$-smoothness and (11))}}\notag    
	\end{align}
	
	
	We present \eqref{gapd} for $t, t-1,..., 1+(k-1)\tau$ with multiplying $1, \gamma,..., \gamma^{t-(k-1)\tau-1}$ respectively as the following:
	\begin{align*}
	&\|\widetilde{\textbf{d}}_i(t)\!-\!\textbf{d}_{[k]}(t)\|\!\leq\!\beta\|\widetilde{\textbf{w}}_{i}(t\!-\!1)\!-\!\textbf{w}_{[k]}(t\!-\!1)\|\!+\!\delta_i\\
	&\!+\!\gamma\|\widetilde{\textbf{d}}_i(t-1)\!-\!\textbf{d}_{[k]}(t-1)\|\\
	&\gamma\|\widetilde{\textbf{d}}_i(t-1)\!-\!\textbf{d}_{[k]}(t-1)\|\leq\gamma(\beta\|\widetilde{\textbf{w}}_{i}(t-2)\!-\!\textbf{w}_{[k]}(t\!-\!2)\|\!+\!\delta_i\\
	&\!+\!\gamma\|\widetilde{\textbf{d}}_i(t\!-\!2)\!-\!\textbf{d}_{[k]}(t\!-\!2)\|)\\
	&...\\
	&\gamma^{t-(k-1)\tau-1}\|\widetilde{\textbf{d}}_i(1+(k-1)\tau)-\textbf{d}_{[k]}(1+(k-1)\tau)\|\\
	&\leq\gamma^{t-(k-1)\tau-1}(\beta\|\widetilde{\textbf{w}}_{i}((k-1)\tau)-\textbf{w}_{[k]}((k-1)\tau)\|\\
	&+\delta_i+\gamma\|\widetilde{\textbf{d}}_i((k-1)\tau)-\textbf{d}_{[k]}((k-1)\tau)\|).
	\end{align*} 
	
	We define $G_i(t)\triangleq\|\widetilde{\textbf{w}}_{i}(t)-\textbf{w}_{[k]}(t)\|$ for convenience. Accumulating all of the above inequalities, we have
	\begin{align*}
	&\|\widetilde{\textbf{d}}_i(t)-\textbf{d}_{[k]}(t)\|\leq\beta(G_i(t-1)+\gamma G_i(t-2)+\gamma^2 G_i(t-3)\\
	&\!+\!...\!+\!\gamma^{t\!-(k\!-\!1)\tau\!-\!1}G_i((k\!-\!1)\tau)+\delta_i(1\!+\!\gamma\!+\!\gamma^{2}\!+\!...\!+\!\gamma^{t\!-\!(k\!-\!1)\tau\!-\!1})\\
	&+\gamma^{t-(k-1)\tau}\|\widetilde{\textbf{d}}_i((k-1)\tau)-\textbf{d}_{[k]}((k-1)\tau)\|.
	\end{align*}
	When $t=(k-1)\tau$, $\textbf{d}_{[k]}((k-1)=\textbf{d}((k-1)\tau)$ according to the definition and $\widetilde{\textbf{d}}_i((k-1)\tau)=\textbf{d}((k-1)\tau)$ by the aggregation rules. So $\|\widetilde{\textbf{d}}_i((k-1)\tau)-\textbf{d}_{[k]}((k-1)\tau)\|=0$ and the above inequality can be further expressed as
	\begin{align}\label{m4}
	&\|\widetilde{\textbf{d}}_i(t)\!-\!\textbf{d}_{[k]}(t)\|\leq\beta(G_i(t\!-\!1)\!+\!\gamma G_i(t\!-\!2)+\gamma^2 G_i(t\!-\!3)\notag\\
	&\!+\!...\!+\!\gamma^{t\!-\!(k\!-\!1)\tau\!-\!1}G_i((k\!-\!1)\tau)
	\!+\!\delta_i(\!1\!+\!\gamma\!+\!\gamma^{2}\!+\!...\!+\!\gamma^{t\!-(k\!-\!1)\tau\!-\!1}).
	\end{align}   
	
	\textbf{We now derive the closed-form expression of the gap between $ \widetilde{\textbf{w}}_i(t)$ and $\textbf{w}_{[k]}(t)$. }For $t\in((k-1)\tau,k\tau]$,
	we have
	\begin{align}
	&\|\widetilde{\textbf{w}}_{i}(t)-\textbf{w}_{[k]}(t)\|\notag\\
	&=\|\widetilde{\textbf{w}}_{i}(t-1)-\eta \widetilde{\textbf{d}}_i(t)-(\textbf{w}_{[k]}(t-1)-\eta \textbf{d}_{[k]}(t))\|\notag\\
	&\rightline{\text{(from (7) and (14))}}\notag\\
	&\label{mmmm}\leq\|\widetilde{\textbf{w}}_{i}(t-1)-\textbf{w}_{[k]}(t-1)\|+\eta\|\widetilde{\textbf{d}}_i(t)-\textbf{d}_{[k]}(t)\|.\\
	&\rightline{\text{(from triangle inequality)}}\notag
	\end{align}
	Substituting inequality \eqref{m4} into \eqref{mmmm} and using $G_i(t)$ to denote $\|\widetilde{\textbf{w}}_{i}(t)-\textbf{w}_{[k]}(t)\|$ for $t, t-1,..., 1+(k-1)\tau$, we have  
	\begin{align}\label{m3}
	&G_i(t)\!\leq\! G_i(t\!-\!1)\!+\!\eta\beta(G_i(t\!-\!1)\!+\!\gamma G_i(t\!-\!2)\!+\!\gamma^2 G_i(t\!-\!3)\!+\!...\!\notag\\
	&\!+\!\gamma^{t\!-\!(k\!-\!1)\tau\!-\!1}G_i((k\!-\!1)\tau))
	\!+\!\eta\delta_i(1\!+\!\gamma\!+\!\gamma^{2}\!+\!...\!+\!\gamma^{t\!-\!(k\!-\!1)\tau\!-\!1}).
	\end{align}
	
	We suppose that 
	\begin{align}
	g_i(t)&=\frac{\delta_i}{\beta}(CA^{t}+DB^{t})\\
	\label{f_i}f_i(t)&=\gamma^{t}g_i(t)-\frac{\delta_i}{\beta}
	\end{align}
	where $A$ and $B$ are expressed as \eqref{defA} and \eqref{defB} respectively, $C=\frac{A-1}{A-B}$ and $D=\frac{1-B}{A-B}.$

	\textbf{Next we use induction to prove that $G_i(t)\leq f_i(t-(k-1)\tau)$ satisfies inequality \eqref{m3} opportunely.}
	For the induction, we assume that
	\begin{align}\label{m5}
	G_i(p)\leq f_i(p-(k-1)\tau)
	\end{align}
	for some $p\in[(k-1)\tau,t-1]$. Therefore, according to \eqref{m3}, we have 
	\begin{align*}
	&G_i(t)\\
	&\!\leq \!f_i(t\!-\!1\!-\!(k\!-\!1)\tau)\!+\!\eta\beta(f_i(t\!-\!1\!-\!(k\!-\!1)\tau))\\
	&+\gamma f_i(t\!-\!2\!-\!(k\!-\!1)\tau)\!+\!\gamma^2 f_i(t\!-\!3\!-\!(k\!-\!1)\tau)\!+\!...\!\\
	&+\!\gamma^{t\!-\!(k\!-\!1)\tau\!-\!1}f_i(0))\!
	+\!\eta\delta_i(1\!+\!\gamma\!+\!\gamma^{2}\!+\!...\!+\!\gamma^{t\!-\!(k-\!1)\tau\!-1})\\
	&\rightline{\text{(from the induction assumption in \eqref{m5})}}\\
	&=\gamma^{t-1-(k-1)\tau}g_i(t\!-\!1\!-\!(k\!-\!1)\tau)\!-\!\frac{\delta_i}{\beta}\\
	&+\eta\beta\gamma^{t-1-(k-1)\tau}(g_i(t\!-\!1\!-\!(k\!-\!1)\tau)+g_i(t\!-\!2\!-\!(k\!-\!1)\tau)\\
	&+g_i(t\!-\!3\!-\!(k\!-\!1)\tau)\!+\!...\!+\!g_i(0))\\
	&\rightline{\text{(from \eqref{f_i})}}\\
	&=\gamma^{t-(k-1)\tau}g_i(t-(k-1)\tau)-\frac{\delta_i}{\beta}\\
	&\rightline{\text{(from Lemma \ref{lemma2} and $g_i(t)=a_t$)}}\\
	&=f_i(t-(k-1)\tau).
	\end{align*}
	Using the induction, we have proven that $\|\widetilde{\textbf{w}}_{i}(t)-\textbf{w}_{[k]}(t)\|\leq f_{i}(t-(k-1)\tau)$ for all $t\in[(k-1)\tau,k\tau]$.
\end{proof}

\textbf{Now, according to the result of lemma \ref{lemma3}, we extend the difference from $\widetilde{\textbf{w}}_i(t)$ to the global one $\textbf{w}(t)$ 
	to prove Proposition 1. }
\begin{proof}[Proof of Proposition 1]
	\textbf{First of all, we derive the gap between $\textbf{d}(t)$ and $\textbf{d}_{[k]}(t)$.}
	we define 
	\begin{gather}\label{p(t)}
	p(t)\triangleq\gamma ^{t} (CA^{t}+DB^{t}  ) -1.
	\end{gather}
	So we get
	\begin{align}\label{39}
	f_i(t)=\frac{\delta_i}{\beta}p(t).
	\end{align}
	
	From (6) and (8), we have
	\begin{align}\label{dtt_1}
	\textbf{d}(t)=\gamma \textbf{d}(t-1)-\frac{\sum_i |\mathcal{D}_i|\nabla F_i(\widetilde{\textbf{w}}_i(t-1))}{|\mathcal{D}|}.
	\end{align}
	
	For $t\in((k-1)\tau,k\tau]$, we have
	\begin{align}
	&\|\textbf{d}(t)-\textbf{d}_{[k]}(t)\|\notag\\
	&=\|\gamma \textbf{d}(t-1) - \frac{\sum_i |\mathcal{D}_i|\nabla F_i(\widetilde{\textbf{w}}_i(t-1))}{|\mathcal{D}|}\notag\\
	&-\gamma \textbf{d}_{[k]}(t-1)+\nabla F(\textbf{w}_{[k]}(t-1))\|\notag\\
	&\rightline{\text{(from \eqref{dtt_1} and (13))}}\notag\\
	&\leq\gamma\| \textbf{d}(t-1)-\textbf{d}_{[k]}(t-1)\|\notag\\
	&+\left (\frac{\sum_i |\mathcal{D}_i|\|\nabla F_i(\widetilde{\textbf{w}}_i(t-1))-\nabla F_i(\textbf{w}_{[k]}(t-1))\|}{|\mathcal{D}|}\right)\notag\\
	&\!\leq\gamma\| \textbf{d}(t-1)-\textbf{d}_{[k]}(t-1)\|+\!\beta\!\left (\frac{\sum_i |\mathcal{D}_i|f_i(t\!-\!1\!-\!(k\!-\!1)\tau)}{|\mathcal{D}|}\!\right)\notag\\
	&\rightline{\text{(from the $\beta$-smoothness and Lemma \ref{lemma3})}}\notag\\
	&\label{div_d}=\gamma\| \textbf{d}(t-1)-\textbf{d}_{[k]}(t-1)\|+\delta p(t-1-(t-1)\tau).\\
	&\rightline{\text{(from \eqref{39} and (12))}}\notag
	\end{align}
	We present \eqref{div_d} for $t, t-1,..., 1+(k-1)\tau$ with multiplying $1, \gamma,..., \gamma^{t-1-(k-1)\tau}$ respectively as the following:	
	\begin{align*}
	&\|\textbf{d}(t)-\textbf{d}_{[k]}(t)\|\\
	&\leq\gamma\| \textbf{d}(t-1)-\textbf{d}_{[k]}(t-1)\|+\delta p(t-1-(t-1)\tau)\\
	&\gamma\|\textbf{d}(t-1)-\textbf{d}_{[k]}(t-1)\|\\
	&\leq\gamma^2\|\textbf{d}(t-2)-\textbf{d}_{[k]}(t-2)\|+\gamma\delta p(t-2-(t-1)\tau)\\
	&...\\
	&\gamma^{t-1-(k-1\tau)}\|\textbf{d}((k-1)\tau+1)-\textbf{d}_{[k]}((k-1)\tau)\|\\
	&\leq\gamma^{t-(k-1)\tau}\|\textbf{d}((k-1)\tau)\!-\!\textbf{d}_{[k]}((k-1)\tau)\|\!+\!\gamma^{t-(k-1)\tau}\delta p(0).
	\end{align*}
	By summing up the left and right sides of the above inequalities respectively, we have
	\begin{align}
	&\|\textbf{d}(t)-\textbf{d}_{[k]}(t)\|\notag\\
	&\leq\delta(\gamma^{t-1-(k-1)\tau}p(0)+...\notag\\
	&+\gamma p(t-2-(k-1)\tau)+p(t-1-(k-1)\tau))\notag\\
	&\label{div_d1}=\delta\left[\left(\frac{C(\gamma A)^{t_0}}{\gamma(A-1)}+\frac{D(\gamma B)^{t_0}}{\gamma (B-1)}\right)-\frac{\gamma^{t_0}-1}{\gamma-1}\right] 
	\end{align}  
	where $t_0=t-(k-1)\tau$.
	
	\textbf{Then we can derive the upper bound between $\textbf{w}(t)$ and $\textbf{w}_{[k]}(t)$ by \eqref{div_d1}.}
	From (7), (8), (9) and (14), we have
	\begin{align*}
	&\|\textbf{w}(t)-\textbf{w}_{[k]}(t)\|\\
	&=\|\textbf{w}(t-1)-\eta \textbf{d}(t)-\textbf{w}_{[k]}(t-1)+\eta \textbf{d}_{[k]}(t)\|\\
	&\leq\|\textbf{w}(t-1)-\textbf{w}_{[k]}(t-1)\|+\eta\|\textbf{d}(t)-\textbf{d}_{[k]}(t)\|.
	\end{align*} 
	According to \eqref{div_d1}, 
	\begin{align}
	&\|\textbf{w}(t)-\textbf{w}_{[k]}(t)\|-\|\textbf{w}(t-1)-\textbf{w}_{[k]}(t-1)\|\notag\\
	&\label{fini}\leq\eta\delta\left[\left(\frac{C(\gamma A)^{t_0}}{\gamma(A-1)}+\frac{D(\gamma B)^{t_0}}{\gamma (B-1)}\right)-\frac{\gamma^{t_0}-1}{\gamma-1}\right]. 
	\end{align}
	When $t=(k-1)\tau$, we have $\|\textbf{w}(t)-\textbf{w}_{[k]}(t)\|=0$ according to the definition. By summing up \eqref{div_d1} over $t\in((k-1)\tau, k\tau]$, we have
	\begin{align*}
	&\|\textbf{w}(t)-\textbf{w}_{[k]}(t)\|\\
	&\leq\sum_{i=1}^{t_0}\eta\delta\left[\left(\frac{C(\gamma A)^{i}}{\gamma(A-1)}+\frac{D(\gamma B)^{i}}{\gamma (B-1)}\right)-\frac{\gamma^{i}-1}{\gamma-1}\right]\\
	&\!=\!\eta\delta\!\left[E((\gamma A)^{t_0}\!-\!1)\!+\!F((\gamma B)^{t_0}\!-\!1)\!-\!\frac{\gamma(\gamma^{t_0}\!-\!1)\!-\!(\gamma\!-\!1)t_0}{(\gamma-1)^2}\right]\\
	&=\!\eta\delta\left[E (\gamma A)^{t_0}\!+\!F (\gamma B)^{t_0}\!-\!\frac{1}{\eta\beta}\!-\!\frac{\gamma(\gamma^{t_0}-1)-(\gamma-1)t_0}{(\gamma-1)^2}\right]\\
	&=h(t_0) 
	\end{align*}
	where $E=\frac{A}{(A-B)(\gamma A-1)}$, $F=\frac{B}{(A-B)(1-\gamma B)}$ and $t_0=t-(k-1)\tau$. The last equality is because $E+F=\frac{1}{\eta\beta}$.

\end{proof}

\section{Proof of increasing of $h(x)$ }\label{C}
\textbf{Firstly, we introduce the following lemma.}
\begin{lemma}\label{jensen}
	Given $A$, $B$, $C$ and $D$ according to their definitions, then the following inequality  
	$$C(\gamma A)^i+D(\gamma B)^i\geq(1+\eta\beta)^i$$
	holds for $i=0,1,2,3,...$.
\end{lemma}
\begin{proof}
	Note that $C+D=1$ and $C>0, D>0$. When $i=0$, $C(\gamma A)^i+D(\gamma B)^i=(1+\eta\beta)^i=0$ and the inequality holds. When $i=1$, 
	\begin{align*}
	&C(\gamma A)^i+D(\gamma B)^i\\
	=&\gamma (CA+DB)\\
	=&\gamma\left(\frac{A-1}{A-B}A+\frac{1-B}{A-B}B\right)\\
	&\rightline{\text{(from the definition of $C,D$)}}\\
	=&\gamma (A+B-1)\\
	&\rightline{\text{(from $A+B=\frac{1+\gamma+\eta\beta}{\gamma}$)}}\\
	=&1+\eta\beta.
	\end{align*}
	The equality still  holds.
	When $i>1$($i$ is an integer), according to Jensen's inequality, we have
	\begin{align*}
	&C(\gamma A)^i+D(\gamma B)^i\\
	>& (\gamma CA+\gamma DB)^i\\
	&\rightline{\text{(because $f(x)=x^i$ is concave)}}\\
	=&(1+\eta\beta)^i.\\
	&\rightline{\text{(from $i=1$)}}
	\end{align*}
	In summary, we have proven the above lemma.
\end{proof}
\textbf{Then we start to proof $h(x)$ increases with $x$.}
\begin{proof}[Proof of increasing of h(x)]
	That $h(x)$ increases with $x$ is equivalent to $$h(x)-h(x-1)\geq0$$ for $x\geq1$ ($x$ is an integer).
	Considering $x=0$ or 1, we have
	\begin{align}
	\label{h0=0}&h(0)=\eta\delta(E+F-\frac{1}{\eta\beta})=0\\
	\label{h1=0}&h(1)=\eta\delta\left[\gamma(EA+FB)-\frac{1}{\eta\beta}-1\right]=0.\\
	&\rightline{\text{(because $EA+FB=\frac{1+\eta\beta}{\gamma\eta\beta}$ which can be derived easily)}}\notag
	\end{align}
	So when $x=1$, $h(x)-h(x-1)=0$.
	
	When $x>1$, we first prove $p(x)\geq 0$ as follows. According to Lemma \ref{jensen}, 
	\begin{align}
	p(x)&=C(\gamma A)^x+D(\gamma B)^x-1\notag \\
	&\geq(1+\eta\beta)^x-1\notag\\
	&\label{p>0}\geq 0.
	\end{align}
	Then we have
	\begin{align*}
	&h(x)-h(x-1)\\
	&=\eta\delta\left[\left(\frac{C(\gamma A)^{x}}{\gamma(A-1)}+\frac{D(\gamma B)^{x}}{\gamma (B-1)}\right)-\frac{\gamma^{x}-1}{\gamma-1}\right]\\
	&=\eta\delta(\gamma^{x-1}p(0)+...+\gamma p(x-2)+p(x-1))\\
	&\rightline{\text{(from \eqref{div_d1} and $x=t-(k-1)\tau$)}}\\
	&\geq 0.\\
	&\rightline{\text{(because $\gamma>0$ and $p(x)\geq0$ according to \eqref{p>0})}}
	\end{align*}

	Therefore, $h(x)$ increases with $x$. 
\end{proof}

\section{Proof of Lemma 2}\label{D}

\begin{proof}[Proof of Lemma 2]
	\textbf{To prove Lemma 2, we present some necessary definitions firstly}.
	
	Given $1\leq k\leq K$ where $K=\frac{T}{\tau}$, we define  
	$$c_{[k]}(t)\triangleq F(\textbf{w}_{[k]}(t))-F(\textbf{w}^*). $$	
	According to the lower bound of any gradient methods for a $\beta$-smooth convex function from \citep[Theorem 3.14]{bubeck2015convex}, we have $$c_{[k]}(t)>0$$ for any $k$ and $t$. 
	According to the definition of $\theta_{[k]}(t)$, $\theta_{[k]}(t)$ denotes the angle between vector $\nabla F(\textbf{w}_{[k]}(t)$ and $\textbf{d}_{[k]}(t)$ so that
	$$\cos\theta_{[k]}(t)=\frac{\nabla F(\textbf{w}_{[k]}(t))^\mathrm{T}\textbf{d}_{[k]}(t)}{\|\nabla F(\textbf{w}_{[k]}(t)\|\|\textbf{d}_{[k]}(t)\|}$$ for $t\in[k]$. We assume $\cos\theta_{[k]}(t)\geq0$ for $1\leq k\leq K$ with $t\in[k]$ and 
	define $\theta=\max_{1\leq k\leq K,t\in[k]}\theta_{[k]}(t)$.
	From the definition of $p$, we have 
	$$p=\max_{1\leq k\leq K,t\in[k]}\frac{\|\textbf{d}_{[k]}(t)\|}{\|\nabla F(\textbf{w}_{[k]}(t))\|}.$$
	Because condition 1 of Lemma 2 and Assumption 1, the sequence generated by MGD converges to the optimum point \citep[Theorem 9]{polyak1964some}.
	
	\textbf{Then we derive the upper bound of $c_{[k]}(t+1)-c_{[k]}(t)$.} 
	
	Because $F(\cdot)$ is $\beta$-smooth, we have 
	\begin{align*}
	F(\textbf{x})-F(\textbf{y})\leq \nabla F(\textbf{y})^\mathrm{T}(\textbf{x}-\textbf{y})+\frac{\beta}{2}\|\textbf{x}-\textbf{y}\|^2
	\end{align*}
	according to \citep[Lemma 3.4]{bubeck2015convex}. Then,
	\begin{align}
	c_{[k]}&(t+1)-c_{[k]}(t)\notag\\
	&=F(\textbf{w}_{[k]}(t+1))-F(\textbf{w}_{[k]}(t))\notag\\
	&\leq\nabla F(\textbf{w}_{[k]}(t))^\mathrm{T}(\textbf{w}_{[k]}(t+1)-\textbf{w}_{[k]}(t))\notag\\
	&+\frac{\beta}{2}\|\textbf{w}_{[k]}(t+1)-\textbf{w}_{[k]}(t)\|^2\notag\\
	&=-\eta\nabla F(\textbf{w}_{[k]}(t))^\mathrm{T}\textbf{d}_{[k]}(t+1)+\frac{\beta\eta^2}{2}\|\textbf{d}_{[k]}(t+1)\|^2\notag\\
	&=-\eta\nabla F(\textbf{w}_{[k]}(t))^\mathrm{T}(\gamma \textbf{d}_{[k]}(t)+\nabla F(\textbf{w}_{[k]}(t)))\notag\\
	&+\frac{\beta\eta^2}{2}\|\gamma \textbf{d}_{[k]}(t)+\nabla F(\textbf{w}_{[k]}(t))\|^2\notag\\
	&=-\eta(1-\frac{\beta\eta}{2})\|\nabla F(\textbf{w}_{[k]}(t))\|^2+\frac{\beta\eta^2\gamma^2}{2}\|\textbf{d}_{[k]}(t)\|^2\notag\\
	&-\eta\gamma(1-\beta\eta)\nabla F(\textbf{w}_{[k]}(t))^{T}\textbf{d}_{[k]}(t)\notag\\
	&\label{c_[k]}\leq\!(-\!\eta(1\!-\!\frac{\beta\eta}{2})\!+\!\frac{\beta\eta^2\gamma^2 p^2}{2}\!\notag\\
	&-\!\eta\gamma(1\!-\!\beta\eta)\cos\theta\!)\|\nabla F(\textbf{w}_{[k]}(t))\|\!^2
	\end{align}
	where the second term in \eqref{c_[k]} is because
	\begin{align*}
	&\|\textbf{d}_{[k]}(t)\|\leq p\|\nabla F(\textbf{w}_{[k]}(t))\|
	\end{align*}
	according to the definition of $p$. Because $\eta\beta\in(0,1)$, the third term in \eqref{c_[k]} is because 
	\begin{align*}
	&\nabla F(\textbf{w}_{[k]}(t))^\mathrm{T}\textbf{d}_{[k]}(t)\\
	&=\|\nabla F(\textbf{w}_{[k]}(t))\|\|\textbf{d}_{[k]}(t)\|\cos\theta_{[k]}(t)\\
	&=\|\nabla F(\textbf{w}_{[k]}(t))\|\|\gamma \textbf{d}_{[k]}(t-1)+\nabla F(\textbf{w}_{[k]}(t-1)\|\cos\theta_{[k]}(t)\\
	&\geq\|\nabla F(\textbf{w}_{[k]}(t))\|\|\nabla F(\textbf{w}_{[k]}(t-1)\|\cos\theta_{[k]}(t)\\
	&\rightline{\text{(because of $\cos\theta_{[k]}(t-1)>0$)}}\\
	&\geq\|\nabla F(\textbf{w}_{[k]}(t))\|^2\cos\theta_{[k]}(t)\\
	&\rightline{\text{(MGD converges \citep[Theorem 9]{polyak1964some})}}\\
	&\geq\|\nabla F(\textbf{w}_{[k]}(t))\|^2\cos\theta.
	\end{align*}

	
	
	According to the definition of $\alpha$, we get 
	$$\alpha=\!\eta(1\!-\!\frac{\beta\eta}{2})\!-\!\frac{\beta\eta^2\gamma^2 p^2}{2}\!
	+\!\eta\gamma(1\!-\!\beta\eta)\cos\theta.$$
	Due to condition 4 of Lemma 2 that $\omega\alpha-\frac{\rho h(\tau)}{\tau\varepsilon^2}>0$, we have  $\alpha>0$.
	Then from \eqref{c_[k]}, we have
	\begin{align}\label{ck}
	c_{[k]}(t+1)\leq c_{[k]}(t)-\alpha \|\nabla F(\textbf{w}_{[k]}(t))\|^2.
	\end{align} 
	
	\textbf{Using this upper bound of $c_{[k]}(t+1)-c_{[k]}(t)$ in \eqref{ck}, we can attain \eqref{c[k]} which can be summed up over $t=0,1,2,...,T-1$ easily. }
	
	According to the convexity condition, 
	\begin{align*}
	c_{[k]}(t)&=F(\textbf{w}_{[k]}(t))-F(\textbf{w}^*)\leq \nabla F(\textbf{w}_{[k]}(t))^\mathrm{T}(\textbf{w}_{[k]}(t)-\textbf{w}^*)\\
	&\leq\|\nabla F(\textbf{w}_{[k]}(t))\|\|\textbf{w}_{[k]}(t)-\textbf{w}^*\|
	\end{align*}
	where the last inequality is from the Cauchy-Schwarz inequality.
	Thus, 
	\begin{align}\label{fv}
	\|\nabla F(\textbf{w}_{[k]}(t))\|\geq\frac{c_{[k]}(t)}{\|\textbf{w}_{[k]}(t)-\textbf{w}^*\|}.
	\end{align}
	Substituting \eqref{fv} to \eqref{ck}, we get
	\begin{align*}
	c_{[k]}(t+1)&\leq c_{[k]}(t)-\frac{\alpha c_{[k]}(t)^2}{\|\textbf{w}_{[k]}(t)-\textbf{w}^*\|^2}\\
	&\leq c_{[k]}(t)-\omega\alpha c_{[k]}(t)^2
	\end{align*}
	where we define 
	\begin{align*}
	\omega\triangleq\min_k\frac{1}{\|\textbf{w}_{[k]}((k-1)\tau)-\textbf{w}^*\|^2}.
	\end{align*}

	Due to the convergence of MGD \citep[Theorem 9]{polyak1964some}, we have $\|\textbf{w}_{[k]}((k-1)\tau)-\textbf{w}^*\|\geq\|\textbf{w}_{[k]}(t)-\textbf{w}^*\|$.
	Hence, $\omega\leq\frac{1}{\|\textbf{w}_{[k]}(t)-\textbf{w}^*\|^2}$.
	So we can get the last inequality.
	
	Because $c_{[k]}(t)>0$, $c_{[k]}(t)c_{[k]}(t+1)>0$. The above inequality is divided by $c_{[k]}(t)c_{[k]}(t+1)$ on both sides. Then we get
	\begin{align*}
	\frac{1}{c_{[k]}(t)}\leq\frac{1}{c_{[k]}(t+1)}-\frac{\omega\alpha c_{[k]}(t)}{c_{[k]}(t+1)}.
	\end{align*}
	From \eqref{ck}, $c_{[k]}(t)\geq c_{[k]}(t+1)$. So we get
	\begin{align}\label{c[k]}
	\frac{1}{c_{[k]}(t+1)} -\frac{1}{c_{[k]}(t)}\geq\frac{\omega\alpha c_{[k]}(t)}{c_{[k]}(t+1)}\geq \omega\alpha.
	\end{align}
	
	\textbf{Finally, we derive the final result of Lemma 2 by summing up \eqref{c[k]} over all intervals and different values of $t$.}
	
	For $(k-1)\tau\leq t\leq k\tau$, 
	\begin{align*}
	&\frac{1}{c_{[k]}(k\tau)}-\frac{1}{c_{[k]}((k-1)\tau)}\\
	=&\sum_{t=(k-1)\tau}^{k\tau-1}\left(\frac{1}{c_{[k]}(t+1)} -\frac{1}{c_{[k]}(t)}\right)\\
	\geq&\sum_{t=(k-1)\tau}^{k\tau-1}\omega\alpha\\
	=&\tau\omega\alpha.
	\end{align*}
	Note that $T=K\tau$. For all $k=1,2,3,...,K$, summing up the above inequality, we have
	\begin{align}
	&\sum_{k=1}^{K}\left(\frac{1}{c_{[k]}(k\tau)}-\frac{1}{c_{[k]}((k-1)\tau)}\right)\notag\\
	=&\frac{1}{c_{[K]}(T)}-\frac{1}{c_{[1]}(0)}-\sum_{k=1}^{K-1}\left(\frac{1}{c_{[k+1]}(k\tau)}-\frac{1}{c_{[k]}(k\tau)}\right)\notag\\
	\geq& K\tau\omega\alpha	=\label{45}T\omega\alpha.
	\end{align}
	
	Considering the right summing terms, 
	\begin{align}
	&\frac{1}{c_{[k+1]}(k\tau)}-\frac{1}{c_{[k]}(k\tau)}\notag\\
	=&\frac{c_{[k]}(k\tau)-c_{[k+1]}(k\tau)}{c_{[k]}(k\tau)c_{[k+1]}(k\tau)}\notag\\
	=&\frac{F(\textbf{w}_{[k]}(k\tau))-F(\textbf{w}_{[k+1]}(k\tau))}{c_{[k]}(k\tau)c_{[k+1]}(k\tau)}\notag\\
	=&\frac{F(\textbf{w}_{[k]}(k\tau))-F(\textbf{w}(k\tau))}{c_{[k]}(k\tau)c_{[k+1]}(k\tau)}\notag\\
	&\rightline{\text{(from the definition $\textbf{w}_{[k+1]}(k\tau)=\textbf{w}(k\tau)$)}}\notag\\
	\label{46}\geq&\frac{-\rho h(\tau)}{c_{[k]}(k\tau)c_{[k+1]}(k\tau)}.\\
	&\rightline{\text{(from Proposition 1)}}\notag
	\end{align}
	As we assume that $F(\textbf{w}_{[k]}(k\tau))-F(\textbf{w}^*)\geq\varepsilon$ for all $k$ and $F(\textbf{w}_{[k]}(t))>F(\textbf{w}_{[k]}(t+1))$ from \eqref{ck}, we
	have $F(\textbf{w}_{[k+1]}(k\tau))-F(\textbf{w}^*)\geq F(\textbf{w}_{[k+1]}((k+1)\tau))-F(\textbf{w}^*)\geq\varepsilon$. Thus, we derive 
	\begin{align}\label{47}
	c_{[k]}(k\tau)c_{[k+1]}(k\tau)\geq\varepsilon^2.
	\end{align}
	Combining \eqref{47} with \eqref{46}, we substitute the combination into \eqref{45} to get
	\begin{align}
	\frac{1}{c_{[K]}(T)}-\frac{1}{c_{[1]}(0)}&\geq -\sum_{k=1}^{K-1}\frac{\rho h(\tau)}{\varepsilon^2}+T\omega\alpha\notag\\
	&\label{48}=T\omega\alpha-(K-1)\frac{\rho h(\tau)}{\varepsilon^2}.
	\end{align}
	Because $T$ is the limited number of local updates, it concerns to the power constraint, not the mathematical constraint. So extending the definition form $K$ to $K+1$ intervals, we can get $F(\textbf{w}(T))-F(\textbf{w}^*)=F(\textbf{w}_{[K+1]}(K\tau))-F(\textbf{w}^*)\geq F(\textbf{w}_{[K+1]}((K+1)\tau))-F(\textbf{w}^*)\geq\varepsilon$.
	According to \eqref{47} for $k=K$, we have
	\begin{align}\label{49}
	(F(\textbf{w}(T))-F(\textbf{w}^*))c_{[K]}(T)>\varepsilon^2.
	\end{align}
	Also according to \eqref{46} for $k=K$, we have
	\begin{align}\label{50}
	\frac{1}{F(\textbf{w}(T))-F(\textbf{w}^*)}-\frac{1}{c_{[K]}(T)}\geq &\frac{-\rho h(\tau)}{(F(\textbf{w}(T))-F(\textbf{w}^*))c_{[K]}(T)}\notag\\
	\geq& - \frac{\rho h(\tau)}{\varepsilon^2}
	\end{align}
	where the last inequality is from \eqref{49}. Summing up \eqref{50} and \eqref{48}, we get
	\begin{align*}
	\frac{1}{F(\textbf{w}(T))-F(\textbf{w}^*)}-\frac{1}{c_{[1]}(0)}\geq T\omega\alpha-K\frac{\rho h(\tau)}{\varepsilon^2}.
	\end{align*}
	Note that $c_{[1]}(0)>0$ and $K=\frac{T}{\tau}$, we have
	\begin{align}\label{51}
	\frac{1}{F(\textbf{w}(T))-F(\textbf{w}^*)}\geq T\left(\omega\alpha-\frac{\rho h(\tau)}{\tau\varepsilon^2}\right).
	\end{align}
	Due to the assumption that $\omega\alpha-\frac{\rho h(\tau)}{\tau\varepsilon^2}>0$, we take the reciprocal of \eqref{51}. The final result follows that 
	\begin{align*}
	F(\textbf{w}(T))-F(\textbf{w}^*)\leq\frac{1}{T\left(\omega\alpha-\frac{\rho h(\tau)}{\tau\varepsilon^2}\right)}.
	\end{align*}
\end{proof}

\section{Proof of Proposition 2}\label{E}
\begin{proof}
	Given $\cos\theta\geq0$, $0<\eta\beta<1$ and $0\leq\gamma<1$, condition $1$ in Lemma 2 always holds. 
	
	Considering $\rho h(\tau)\!=\!0$, because $\alpha>0$, condition $4$ in Lemma 2 is satisfied. By choosing an arbitrarily small $\varepsilon$, conditions $2$ and $3$ in Lemma 2 can be satisfied. Thus, according to the definition of $\textbf{w}^\mathrm{f}$ and Lemma 2, we have
	$$F(\textbf{w}^\mathrm{f})-F(\textbf{w}^*)\leq F(\textbf{w}(T))-F(\textbf{w}^*)\leq \frac{1}{T\omega\alpha}.$$
	So we prove Proposition 2 when $\rho h(\tau)=0$.
	
	Considering $\rho h(\tau)>0$, we set 
	\begin{align}\label{53}
	\varepsilon_0=\frac{1}{T\left(\omega\alpha-\frac{\rho h(\tau)}{\tau\varepsilon_0^2}\right)}.
	\end{align}
	After solving the above equation, we get
	\begin{align}\label{54}
	\varepsilon_0=\frac{1}{2T\omega\alpha}+\sqrt{\frac{1}{4T^2\omega^2\alpha^2 }+\frac{\rho h(\tau)}{\omega\alpha\tau}}
	\end{align}
	where because $\varepsilon>0$, we remove the negative solution and $\varepsilon_0>0$. For any $\varepsilon>\varepsilon_0$, it holds that
	$$\omega\alpha-\frac{\rho h(\tau)}{\tau\varepsilon^2}\geq\omega\alpha-\frac{\rho h(\tau)}{\tau\varepsilon_0^2}.$$
	Because $\omega\alpha-\frac{\rho h(\tau)}{\tau\varepsilon_0^2}=\frac{1}{\varepsilon_0 T}>0$ from \eqref{53} and \eqref{54}, condition $4$ in Lemma 2 is satisfied.
	Assume that $\exists \varepsilon>\varepsilon_0$ satisfies condition $2$ and $3$ in Lemma 2, then all conditions have been satisfied. So we have 
	$$F(\textbf{w}(T))-F(\textbf{w}^*)\!\leq\!\frac{1}{T\left(\omega\alpha-\frac{\rho h(\tau)}{\tau\varepsilon^2}\right)}\!\leq\!\frac{1}{T\left(\omega\alpha-\frac{\rho h(\tau)}{\tau\varepsilon_0^2}\right)}\!=\!\varepsilon_0.$$
	We derive $F(\textbf{w}(T))-F(\textbf{w}^*)<\varepsilon$, that contradicts with condition $3$ in Lemma 2. Hence, we know that condition $2$ and $3$ can not be satisfied at the same time for $\varepsilon>\varepsilon_0$. That means either $\exists k$ so that $F(\textbf{w}_{[k]}(k\tau))-F(\textbf{w}^*)\leq\varepsilon_0$ or $F(\textbf{w}(T))-F(\textbf{w}^*)\leq\varepsilon_0$. We combine this two cases and get
	\begin{align}\label{55}
	\min\left\{\min_{k} F(\textbf{w}_{[k]}(k\tau)); F(\textbf{w}(T)) \right\}-F(\textbf{w}^*)\leq\varepsilon_0.
	\end{align}
	According to Proposition 1, $F(\textbf{w}_{[k]}(k\tau))\geq F(\textbf{w}(k\tau))-\rho h(\tau)$.
	So \eqref{55} is rewritten as
	$$\min\left\{\min_{k} F(\textbf{w}(k\tau))-\rho h(\tau); F(\textbf{w}(T)) \right\}-F(\textbf{w}^*)\leq\varepsilon_0.$$
	When $k=K$, $F(\textbf{w}(K\tau))-\rho h(\tau)=F(\textbf{w}(T))-\rho h(\tau)\leq F(\textbf{w}(T))$. Combining the definition of $\textbf{w}^\mathrm{f}$, we obtain
	$$F(\textbf{w}^\mathrm{f})-F(\textbf{w}^*)\leq \varepsilon_0+\rho h(\tau)$$ that is the result in Proposition 2.
\end{proof}
\section{}\label{hospital}
\begin{lemma}\label{lem_hospital}
	$\lim_{\tau\to0}h(\tau)=0$.	
\end{lemma}
\begin{proof}
	If $\tau\to0$, we can get $\lim_{\tau\to0}A=\frac{1}{r}$, $\lim_{\tau\to0}B=1$ and $\lim_{\tau\to0}F=\frac{\gamma}{(1-\gamma)^2}$ straightly.
	Then we have
	\begin{align*} 
	&\lim_{\tau\to0}h(\tau)\\
	=&\lim_{\tau\to0}\!\eta\delta\left[ E (\gamma A)^\tau\!+\!F (\gamma B)^\tau\!-\!\frac{1}{\eta\beta}\!-\!\frac{\gamma(\gamma^\tau-1)-(\gamma-1)\tau}{(\gamma-1)^2}\right]\\
	=&\lim_{\tau\to0}\!\eta\delta\left( E -\!\frac{1}{\eta\beta}\right)=\frac{\delta}{1-\gamma}\lim_{\tau\to0}\frac{\eta}{\gamma A-1}-\frac{\delta}{\beta}=0
	\end{align*}
	where the last equality is because by using L'Hospital's rule, we have
	$$\lim_{\tau\to0}\frac{\eta}{\gamma A-1}=\frac{1-\gamma}{\beta}.$$
\end{proof}
\end{document}